\documentclass{article}

\usepackage{authblk}
\usepackage{fullpage}
\usepackage{hyperref}       
\usepackage{url}            
\usepackage{booktabs}       
\usepackage{amsfonts}       
\usepackage{nicefrac}       
\usepackage[]{subcaption}

\usepackage{xcolor}
\usepackage{graphicx}
\usepackage{amsmath,amssymb}
\usepackage{amsthm}
\usepackage{bbm}
\usepackage{bm}
\usepackage{mathrsfs}
\usepackage{appendix}
\usepackage{graphicx}
\usepackage{pifont}
\usepackage{cleveref}
\usepackage{xcolor}
\usepackage{enumitem}

\graphicspath{{figures/}}
\usepackage{xcolor}
\usepackage[square, comma, numbers,sort&compress]{natbib}

\usepackage{mdframed}
\usepackage{comment}
\newtheorem{thm}{Theorem}
\newtheorem{lemma}{Lemma}
\newtheorem{definition}{Definition}

\newtheorem{rmk}{Remark}
\newtheorem{fact}{Fact}
\newtheorem{clm}{Claim}

\newmdtheoremenv{model}{Model}

\usepackage{algorithm}
\usepackage{algorithmicx}
\usepackage{algpseudocode}

\algtext*{EndWhile}
\algtext*{EndFor}
\algtext*{EndIf}
\algnewcommand\algorithmicinput{\textbf{Input:}}
\algnewcommand\INPUT{\item[\algorithmicinput]}
\algnewcommand\algorithmicoutput{\textbf{Output:}}
\algnewcommand\OUTPUT{\item[\algorithmicoutput]}

\DeclareMathOperator*{\argmin}{argmin}

\DeclareMathOperator{\E}{\mathbb{E}}

\renewcommand{\epsilon}{\varepsilon}

\newcommand{\D}{\mathcal{D}}
\newcommand{\cC}{\mathcal{C}}

\newcommand{\w}{\bold{w}}
\newcommand{\alloc}{\bold{a}}
\newcommand{\cvec}{\bold{c}}
\newcommand{\dvec}{\bold{d}}

\newcommand{\R}{\mathbb{R}}
\newcommand{\ones}{\mathbbm{1}}

\usepackage{accents}

\usepackage[draft,inline,nomargin,index]{fixme}
\fxsetup{theme=color,mode=multiuser,inlineface=\itshape,envface=\itshape}

\newcommand{\saeed}[1]{}
\newcommand{\juba}[1]{}
\newcommand{\hadi}[1]{}
\newcommand{\emily}[1]{}
\newcommand \travis [1]{}
\newcommand{\ar}[1]{}
\newcommand{\zs}[1]{}
\newcommand{\mk}[1]{}

\newcommand \reals {\mathbb{R}}
\newcommand \expect {\operatorname*{\mathbb{E}}}
\newcommand \prob {\operatorname*{Pr}}
\newcommand \ind [1]{\mathbb{I}\{#1\}}
\newcommand \indinf [1]{\mathbb{I}_\infty \{#1\}}
\newcommand \regret {\operatorname{R}}
\newcommand \optregret {\mathcal{R}}
\newcommand \solspace {\mathcal{I}}
\newcommand \cF {\mathcal{F}}

\newcommand \pdim {\operatorname{PDim}}

\newcommand \Rfairdet {\mathcal{R}_{\text{fair}}}
\newcommand \Rfairrand {\widehat{\mathcal{R}}_{\text{fair}}}
\newcommand{\ceilstep}[1]{\operatorname{ceil}_{#1}}

\newcommand\shortversion[1]{}
\newcommand\longversion[1]{#1}

\usepackage{authblk}

\title{Algorithms and Learning for Fair Portfolio Design}

\author[]{Emily Diana, Travis Dick, Hadi Elzayn, Michael Kearns, Aaron Roth}
\author[]{Zachary Schutzman, Saeed Sharifi-Malvajerdi, Juba Ziani}

\affil[]{\textit{University of Pennsylvania}}

\begin{document}

\maketitle

\begin{abstract}
We consider a variation on the classical finance problem of optimal portfolio design. In our setting, a large
population of consumers is drawn from some distribution over {\em risk tolerances\/}, and each consumer
must be assigned to a portfolio of lower risk than her tolerance. The consumers may also belong to underlying
groups (for instance, of demographic properties or wealth), and the goal is to design a small number of
portfolios that are fair across groups in a particular and natural technical sense.

Our main results are algorithms for optimal and near-optimal 
portfolio design for both social welfare and fairness objectives, both with and without assumptions on the underlying group structure.
We describe an efficient algorithm based on an internal two-player zero-sum game that learns
near-optimal fair portfolios {\em ex ante\/} and show experimentally that it can be used to obtain a small set of fair
portfolios {\em ex post\/} as well. For the special but natural case in which group structure coincides with
risk tolerances (which models the reality that wealthy consumers generally tolerate greater risk), we give
an efficient and optimal fair algorithm. We also provide generalization guarantees for the underlying
risk distribution that has no dependence on the number of portfolios and illustrate the theory with
simulation results.
\end{abstract}

\longversion{
\clearpage
\tableofcontents
\clearpage
}

\section{Introduction}

In this work we consider algorithmic and learning problems in a model for
the fair design of financial products. We imagine a large population of individual
retail investors or {\em consumers\/}, each of whom has her own tolerance
for investment risk in the form of a limit on the variance of returns. It is well known
in quantitative finance that for any set of financial assets, the optimal expected returns on
investment are increasing in risk.

A large retail investment firm (such as Vanguard or Fidelity) wishes to design 
portfolios to serve these consumers under the common practice of
assigning consumers only to portfolios with lower risks than their tolerances. 
The firm would like to design and offer only a {\em small\/} number of such products ---
much smaller than the number of consumers --- since the execution and maintenance of
portfolios is costly and ongoing.
The
overarching goal is to design the products to minimize consumer {\em regret\/} ---
the loss of returns due to being assigned to lower-risk portfolios compared to the
bespoke portfolios saturating tolerances --- both with and without fairness
considerations.

We consider a notion of group fairness adapted from the literature on fair division.
Consumers belong to underlying groups that may be defined by standard demographic features such as race, gender or age,
or they may be defined by the risk tolerances themselves, as higher risk appetite is
generally correlated with higher wealth. We study 
{\em minmax group fairness}, in which the goal is to minimize the maximum regret across groups --- i.e. to optimize for the \emph{least well-off} group. Compared to the approach of constraining regret to be equal across groups (which is not even always feasible in our setting), minmax optimal solutions have the property that they Pareto-dominate 
regret-equalizing solutions:
every group has regret that can only be smaller than it would be if regret were constrained to be equalized across groups. 

\paragraph{Related Work.}
Our work generally falls within the literature on fairness in machine learning, which is too broad to survey in detail here; see~\cite{cacmsurvey} for a recent overview. We are perhaps closer in spirit to
research in fair division or allocation problems~\cite{maxmin1,maxmin2,maxmin3}, in which a limited resource must be distributed across a
collection of players so as to maximize the utility of the \emph{least} well-off; here, the resource in question is
the small number of portfolios to design. However, we are not aware of any technical connections between
our work and this line of research. Within the fairness in machine learning literature, our work is closest to \emph{fair facility location} problems~\cite{jung2019center,mahabadi2020individual}, which attempt to choose a small number of ``centers'' to serve a large and diverse population and ``fair allocation'' problems that arise in the context of predictive policing~\cite{Ensign,elzayn2019fair,donahue2020fairness}.

There seems to have been relatively little explicit consideration of fairness issues in quantitative finance
generally and optimal portfolio design specifically. An exception is~\cite{iancu2014fairness},
in which the interest is
in fairly amortizing transaction costs of a single portfolio across investors rather than designing
multiple portfolios to meet a fairness criterion.

\paragraph{Our Results and Techniques.}
In Section~\ref{sec:dp}, we provide a dynamic program to find $p$ products that optimize the average regret of a single population. In Section~\ref{sec:fairdp}, we divide the population into different groups and develop techniques to guarantee minmax group fairness: in Section~\ref{sec:separation}, we show a separation between deterministic solutions and randomized solutions (i.e. distributions over sets of $p$ products) for minmax group fairness; in Section~\ref{sec:exante_fair}, we leverage techniques for learning in games to find a distribution over products that optimizes for minmax group fairness; in Section~\ref{sec:pure_minmax_few_groups}, we focus on deterministic solutions and extend our dynamic programming approach to efficiently optimize for the minmax objective when the number of groups is constant; in Section~\ref{sec:interval_minmax}, we study the natural special case in which groups are defined by disjoint intervals on the real line and give algorithms that are efficient even for large numbers of groups. In Section~\ref{sec:generalization}, we show that when consumers' risk tolerances are drawn i.i.d. from an unknown distribution, empirical risk bounds from a sample of consumers generalize to the underlying distribution, and we prove tight sample complexity bounds. Finally, in Section~\ref{sec:experiments}, we provide experiments to complement our theoretical results. \shortversion{All proofs and technical details are in the full version of the paper, provided in the supplement.}

\section{Model and Preliminaries}\label{sec:model}

We aim to create products (portfolios) consisting of weighted collections of assets with differing means and standard deviations (risks). 
Each consumer is associated with a real number $\tau \in \R_{\ge 0}$, which is the {\em risk tolerance\/} of the consumer --- an upper bound on the standard deviation of returns.
We assume that consumers' risk tolerances are bounded.

\shortversion{\paragraph{Bespoke Problem.}
We adopt the standard Markowitz framework \cite{markowitz1952portfolio} for portfolio design in which the \emph{bespoke} 
portfolio achieves the maximum expected return that the consumer can realize, subject to the constraint that the overall risk not exceed her tolerance $\tau$. 
We defer to the supplement the standard formulation of the bespoke problem, which can be solved by standard convex optimization methods;
in Section \ref{sec:experiments} we provide a detailed example.
Here we only note that the solution is summarized by a function $r(\tau)$, which is non-decreasing (by definition), 
and since consumer tolerances are bounded, we let $B$ denote the maximum value of $r(\tau)$ across all consumers.
}
\longversion{
\subsection{Bespoke Problem} We adopt the standard Markowitz framework \cite{markowitz1952portfolio} for portfolio design. Given a set of $m$ assets with mean $\mu \in \R_+^m$ and covariance matrix $\Sigma \in \R^{m \times m}$, and a consumer risk tolerance $\tau$, the \emph{bespoke} portfolio achieves the maximum expected return that the consumer can realize by assigning weights $\alloc$ over the assets (where the weight of an asset represents the fraction of the portfolio allocated to said asset) subject to the constraint that the overall risk -- quantified as the standard deviation of the mixture $\alloc$ over assets -- does not exceed her tolerance $\tau$. 
Finding a consumer's bespoke portfolio can be written down as an optimization problem. We formalize the bespoke problem in Equation~\eqref{eq:return-function} below, and call the solution $r(\tau)$. 
\begin{equation}\label{eq:return-function}
r (\tau) = \max_{\alloc \in \R^m} \left\{ \alloc^\top \mu \, | \, \alloc^\top \Sigma \, \alloc \le \tau^2, \ones^\top \alloc = 1\right\} 
\end{equation}
Here we note that optimal portfolios are summarized by function $r(\tau)$, which is non-decreasing. Since consumer tolerances are bounded, we let $B$ denote the maximum value of $r(\tau)$ across all consumers.
}

\longversion{\subsection{A Regret Notion}}
\shortversion{\paragraph{A Regret Notion.}}
Suppose there are $n$ consumers to whom we want to offer products. Our goal is to design $p \ll n$ products
that minimize a notion of \emph{regret} for a given set of consumers.
A product has a risk (standard deviation) which we will denote by $c$. 
We assume throughout that in addition to the selected $p$ products, there is always a risk-free product (say cash) available that has zero return;
we will denote this risk-free product by $c_0 \equiv 0$ throughout the paper ($r(c_0) = 0$). 
For a given consumer with risk threshold $\tau$, the regret of the consumer with respect to a set of products $\cvec = (c_1, c_2, \ldots, c_p) \in \R_{\ge 0}^p$ 
is the difference between the return of her bespoke product 
and the maximum return of any product with risk that is 
\emph{less than or equal to} her risk threshold. To formalize this, the regret of products $\cvec$ for a consumer with risk threshold $\tau$ is defined as $\regret_\tau (\cvec) = r(\tau) - \max_{c_j \le \tau} r (c_j)$. Note since $c_0 = 0$ always exists, the $\max_{c_j \le \tau} r (c_j)$ term is well defined. Now for a given set of consumers $S = \{ \tau_i \}_{i=1}^n$, the regret of products $\cvec$ on $S$ is simply defined as the \emph{average regret} of $\cvec$ on $S$:
 \begin{equation}\label{eq:avg-regret}
 \regret_S (\cvec) \triangleq \frac{1}{n} \sum_{i=1}^n \regret_{\tau_i} (\cvec)
 \end{equation}
When $S$ includes the entire population of consumers, we call $\regret_S(\cvec)$ the \emph{population regret}. The following notion for the \emph{weighted regret} of $\cvec$ on $S$, given a vector $\w = (w_1,\ldots,w_n)$ of weights for each consumer, will be useful in Sections~\ref{sec:dp} and \ref{sec:fairdp}:
\shortversion{
$\regret_S (\cvec,\w) = \sum_{i=1}^n w_i \regret_{\tau_i} (\cvec)$.
}
\longversion{
 \begin{equation}\label{eq:weighted-regret}
 \regret_S (\cvec,\w) \triangleq \sum_{i=1}^n w_i \regret_{\tau_i} (\cvec)
 \end{equation}
 }Absent any fairness concern, our goal is to design efficient algorithms to minimize $\regret_S (\cvec)$ for a given set of consumers $S$ and target number of products $p$: $\min_{\cvec} \regret_{S} (\cvec)$. This will be the subject of Section~\ref{sec:dp}. We can always find an optimal set of products as a subset of the $n$ consumer risk thresholds $S = \{\tau_i\}_i$, because if any product $c_j$ is not in $S$, we can replace it by $\min \{ \tau_i \mid \tau_i \geq c_j \}$ without decreasing the return for any consumer\footnote{We consider a more general regret framework in Appendix~\ref{app:two-sided} which allows consumers  to be assigned to products with risk higher than their tolerance, and in this case it is not necessarily optimal to always place products on the consumer risk thresholds.}. We let $C_p (S)$ represent the set of all subsets of size $p$ for a given set of consumers $S$: $C_p (S) = \left\{ \cvec = (c_1, c_2, \ldots, c_p) \subseteq S \right\}$. We therefore can reduce our regret minimization problem to the following problem:
\begin{equation}\label{eq:opt}
\mathcal{R} \left(S,p \right) \triangleq \min_{\cvec \, \in C_p (S)} \regret_{S} (\cvec)
\end{equation}
Similarly, we can reduce the weighted regret minimization problem to the following problem: 
\begin{equation}\label{eq:opt_weighted}
\mathcal{R} \left(S,\w,p \right) \triangleq \min_{\cvec \, \in C_p (S)}  \regret_S (\cvec,\w).
\end{equation}

\longversion{\subsection{Group Fairness: \emph{ex post} and \emph{ex ante}}}
\shortversion{\paragraph{Group Fairness: \emph{ex post} and \emph{ex ante}.}}
Now suppose consumers are partitioned into $g$ groups: $S = \{ G_k \}_{k=1}^g$, e.g. based on their attributes like race or risk levels.  Each $G_k$ consists of the consumers of group $k$ represented by their risk thresholds. We will often abuse notation and write $i \in G_k$ to denote that consumer $i$ has threshold $\tau_i \in G_k$. Given this group structure, minimizing the regret of the whole population absent any constraint might lead to some groups incurring much higher regret than others. With fairness concerns in mind, we turn to the  design of efficient algorithms to minimize the maximum regret over groups (we call this maximum ``group regret''):
\begin{equation}\label{eq:fair-opt-det}
\Rfairdet \left(S,p \right) \triangleq \min_{\cvec \, \in C_p (S)} \left\{ \max_{1 \le k \le g} \regret_{G_k} (\cvec) \right\}
\end{equation}
The set of products $\cvec$ that solves the above minmax problem\footnote{In Appendix~\ref{app:pop_vs_group_optim}, we show that optimizing for population regret can lead to arbitrarily bad group regret in relative terms, and vice-versa.} will be said to satisfy \emph{ex post} minmax fairness (for brevity, we call this ``fairness'' throughout). One can relax problem~\eqref{eq:fair-opt-det} by allowing the designer to \emph{randomize} over sets of $p$ products and output a \emph{distribution} over $C_p (S)$ (as opposed to one deterministic set of products) that minimizes the maximum \emph{expected} regret of groups:
\begin{equation}\label{eq:fair-opt-rand}
\Rfairrand \left(S,p \right) \triangleq \min_{\cC \in \Delta (C_p (S))} \left\{ \max_{1 \le k \le g} \E_{\, \cvec \sim \cC} \left[ \regret_{G_k} (\cvec) \right] \right\}
\end{equation}
where $\Delta (A)$ represents the set of probability distributions over the set $A$, for any $A$. The distribution $\cC$ that solves the above minmax problem will be said to satisfy \emph{ex ante} minmax fairness  --- meaning fairness is satisfied in expectation \emph{before} realizing any set of products drawn from the distribution $\cC$ --- but there is no fairness guarantee on the realized draw from $\cC$. Such a notion of fairness is useful in settings in which the designer has to make repeated decisions over time and has the flexibility to offer different sets of products in different time steps. In Section~\ref{sec:fairdp}, we provide algorithms that solve both problems cast in Equations~(\ref{eq:fair-opt-det}) and (\ref{eq:fair-opt-rand}). We note that while there is a simple integer linear program (ILP) that solves  Equations~(\ref{eq:fair-opt-det}) and (\ref{eq:fair-opt-rand}), such an ILP is often intractable to solve. We use it in our experiments on small instances to evaluate the quality of our efficient algorithms.

\section{Regret Minimization Absent Fairness} \label{sec:dp}

In this section we provide an efficient dynamic programming algorithm for
finding the set of $p$ products that minimizes the (weighted) regret for a
collection of consumers. This dynamic program will be used as a subroutine in our algorithms for finding
optimal products for minmax fairness.

Let $S = \{\tau_i\}_{i=1}^n$ be a collection of consumer risk thresholds and
$\mathbf{w} = (w_i)_{i=1}^n$ be their weights, such that $\tau_1 \leq \dots \leq
\tau_n$ (w.l.o.g). The key idea is as follows: suppose that consumer index $z$ defines the
riskiest product in an optimal set of $p$ products.  Then all consumers $z,
\dots, n$ will be assigned to that product and the consumers $1, \dots, z-1$
will not be. Therefore, if we knew the highest risk product in an optimal
solution, we would be left with a smaller sub-problem in which the goal is to
optimally choose $p-1$ products for the first $z-1$ consumers. Our dynamic
programming algorithm finds the optimal $p'$ products for the first $n'$
consumers for all values of $n' \leq n$ and $p' \leq p$.

\shortversion{
We defer technical details to the supplement and provide the guarantees of our algorithm below:
\begin{thm} \label{thm:dp}
  There exists an algorithm that, given a collection of consumers $S =
  \{\tau_i\}_{i=1}^n$ with weights $\mathbf{w} = (w_i)_{i=1}^n$ and a target
  number of products $p$, outputs a collection of products $\mathbf{c} \in
  C_p(S)$ with minimal weighted regret: $R_S(\mathbf{w}, \mathbf{c}) =
  \optregret(S, \mathbf{w}, p)$. This algorithm runs in time $O(n^2 p)$.
\end{thm}
}

\longversion{
More formally for any $n' \leq n$, let $S[n'] = \{\tau_i\}_{i=1}^{n'}$ and
$\mathbf{w}[n']$ denote the $n'$ lowest risk consumers and their weights. For
any $n' \leq n$ and $p' \leq p$, let $T(n', p') = \optregret(S[n'],
\mathbf{w}[n'], p')$ be the optimal weighted regret achievable in the
sub-problem using $p'$ products for the first $n'$ weighted consumers. We make use of the following recurrence relations:

\begin{lemma}\label{lem:dpRelations}
  The function $T$ defined above satisfies the following properties:
  \begin{enumerate}
  \item For any $1 \leq n' \leq n$, we have $T(n', 0) = \sum_{i=1}^{n'} w_i
  r(\tau_i)$.
  \item For any $1 \leq n' \leq n$ and $0 \leq p' \leq p$, we have
  \[
  T(n', p') =
  \min_{z \in \{p', \dots, n'\}} \left(
  T(z-1, p'-1) + \sum_{i=z}^{n'} w_i\bigl(r(\tau_i) - r(\tau_z)\bigr)
  \right).
  \]
  \end{enumerate}
\end{lemma}

\begin{proof}
See Appendix~\ref{app:dynamic-program}.
\end{proof}

The running time of our dynamic programming algorithm, which uses the above
recurrence relations to solve all sub-problems, is summarized below.

\begin{thm} \label{thm:dp}
  There exists an algorithm that, given a collection of consumers $S =
  \{\tau_i\}_{i=1}^n$ with weights $\mathbf{w} = (w_i)_{i=1}^n$ and a target
  number of products $p$, outputs a collection of products $\mathbf{c} \in
  C_p(S)$ with minimal weighted regret: $\regret_S(\cvec,\mathbf{w}) =
  \optregret(S, \mathbf{w}, p)$. This algorithm runs in time $O(n^2 p)$.
\end{thm}
\begin{proof}
  The algorithm computes a table containing the values $T(n', p')$ for all
  values of $n' \leq n$ and $p' \leq p$ using the above recurrence relations.
  The first column, when $p' = 0$, is computed using property 1 from
  \Cref{lem:dpRelations}, while the remaining columns are filled using property
  2. By keeping track of the value of the index $z$ achieving the minimum in
  each application of property 2, we can also reconstruct the optimal products
  for each sub-problem.

  To bound the running time, observe that the sums appearing in both properties
  can be computed in $O(1)$ time, after pre-computing all partial sums of the
  form $\sum_{i=1}^{n'} w_i r(\tau_i)$ and $\sum_{i=1}^{n'} w_i$ for $n' \leq
  n$. Computing these partial sums takes $O(n)$ time. With this, property 1 can
  be evaluated in $O(1)$ time, and property 2 can be evaluated in $O(n)$ time
  (by looping over the values of $z$). In total, we can fill out all $O(np)$
  table entries in $O(n^2 p)$ time. Reconstructing the optimal set of products
  takes $O(p)$ time.
\end{proof}
}

\section{Regret Minimization with Group Fairness}\label{sec:fairdp}

In this section, we study the problem of choosing $p$ products when the consumers can be partitioned into $g$ groups, and we want to optimize minmax fairness across groups, for both the {\em ex post\/} minmax fairness Program~\eqref{eq:fair-opt-det} and the {\em ex ante\/} minmax fairness Program~\eqref{eq:fair-opt-rand}.

We start the discussion of minmax fairness by showing a separation between the {\em ex post\/} objective in Program~\eqref{eq:fair-opt-det} and the {\em ex ante\/} objective in Program~\eqref{eq:fair-opt-rand}. More precisely, we show in subsection~\ref{sec:separation} that the objective value of Program~\eqref{eq:fair-opt-det} can be $\Omega(g)$ times higher than the objective value of Program~\eqref{eq:fair-opt-det}.

In the remainder of the section, we provide algorithms to solve Programs~\eqref{eq:fair-opt-det} and~\eqref{eq:fair-opt-rand}. In subsection~\ref{sec:exante_fair}, we provide an algorithm that solves Program~\eqref{eq:fair-opt-rand} to any desired additive approximation factor via no-regret dynamics. In subsection~\ref{sec:pure_minmax_few_groups}, we provide a dynamic program approach that finds an approximately optimal solution to Program~\eqref{eq:fair-opt-det} when the number of groups $g$ is small. Finally, in subsection~\ref{sec:interval_minmax}, we provide a dynamic program that solves Program~\eqref{eq:fair-opt-det} exactly in a special case of our problem in which the groups are given by disjoint intervals of consumer risk thresholds. 

\subsection{Separation Between Randomized and Deterministic Solutions}\label{sec:separation}

\shortversion{

The following theorem relates the minmax (expected) regret that can be achieved with deterministic versus randomized solutions (i.e. the solutions to Programs~\eqref{eq:fair-opt-det} and~\eqref{eq:fair-opt-rand}).

First, there is a separation between the minmax regret achievable by deterministic versus randomized strategies: the optimal objective $\Rfairdet$ of the best deterministic strategy can be $\Omega(g)$ times worse than the objective value $\Rfairrand$ of the best randomized strategy. Second, for any instance of our problem, by allowing a multiplicative factor $g$ blow-up in the target number of products $p$, the optimal deterministic minmax value will be at least as good as the randomized minmax value with $p$ products. We defer the proof to the supplement.

\begin{thm}
For any $g$ and $p$, there exists an instance $S$ consisting of $g$ groups such that
$$
\frac{\Rfairrand \left( S, p \right)}{\Rfairdet \left( S, p \right)} \le \frac{1}{p+1} \left\lceil \frac{p+1}{g} \right\rceil.
$$
Further, we have that for any instance $S$ consisting of $g$ groups, and any $p$,
$$
\Rfairdet \left( S, gp \right) \le \Rfairrand \left( S, p \right).
$$
\end{thm}
}

\longversion{
The following theorem shows a separation between the minmax (expected) regret achievable by deterministic versus randomized strategies (as per Programs~\eqref{eq:fair-opt-det} and~\eqref{eq:fair-opt-rand}); in particular, the regret $\Rfairdet$ of the best deterministic strategy can be $\Omega(g)$ times worse than the regret $\Rfairrand$ of the best randomized strategy:
\begin{thm}\label{thm:separation}
For any $g$ and $p$, there exists an instance $S$ consisting of $g$ groups such that
$$
\frac{\Rfairrand \left( S, p \right)}{\Rfairdet \left( S, p \right)} \le \frac{1}{p+1} \left\lceil \frac{p+1}{g} \right\rceil
$$
\end{thm}

\begin{proof}
The proof is provided in Appendix~\ref{app:separation}.
\end{proof}

In the following theorem, we show that for any instance of our problem, by allowing a multiplicative factor $g$ blow-up in the target number of products $p$, the optimal deterministic minmax value will be at least as good as the randomized minmax value with $p$ products.
\begin{thm}
We have that for any instance $S$ consisting of $g$ groups, and any $p$,
$$
\Rfairdet \left( S, gp \right) \le \Rfairrand \left( S, p \right)
$$
\end{thm}

\begin{proof}
Fix any instance $S = \{G_k\}_{k=1}^g$ and any $p$. Let $\cvec_k^* \triangleq \argmin_{\cvec \in C_p (S)} \regret_{G_k} ( \cvec )$ which is the best set of $p$ products for group $G_k$. We have that
\begin{align*}
\Rfairdet \left( S, gp \right) &=  \min_{ \cvec \in C_{gp} (S) } \max_{1 \le k \le g} \regret_{G_k} ( \cvec ) \\
& \le \max_{1 \le k \le g} \regret_{G_k} \left( \cup_k \cvec_k^* \right) \\
&\le  \max_{1 \le k \le g} \regret_{G_k} ( \cvec_k^* ) \\
&\le \max_{1 \le k \le g} \E_{\cvec \sim \cC} \left[ \regret_{G_k} (  \cvec ) \right]
\end{align*}
where the last inequality follows from the definition of $\cvec_k^*$ and it holds for any distribution $\cC \in \Delta (C_p (S))$.
\end{proof}
}

\subsection{An Algorithm to Optimize for {\em ex ante\/} Fairness}\label{sec:exante_fair}

\longversion{
In this section, we provide an algorithm to solve the {\em ex ante\/} Program~\eqref{eq:fair-opt-rand}. Remember that the optimization problem is given by
\begin{equation}\tag{\ref{eq:fair-opt-rand}}
\Rfairrand \left(S,p \right) \triangleq \min_{\cC \in \Delta (C_p (S))} \left\{ \max_{1 \le k \le g} \E_{\, \cvec \sim \cC} \left[ \regret_{G_k}(\cvec) \right] \right\}
\end{equation}
}
\shortversion{
We provide an algorithm to solve the {\em ex ante\/}
Program~\eqref{eq:fair-opt-rand}.} 
Algorithm~\eqref{alg: dynamics} relies on the
dynamics introduced by~\citet{freund96} to solve
Program~\eqref{eq:fair-opt-rand}. The algorithm interprets this
minmax optimization problem as a zero-sum game between the designer, who wants
to pick products to minimize regret, and an adversary, whose goal is to pick
the highest regret group. This game is
played repeatedly, and agents update their strategies at every time step based
on the history of play. In our setting, the adversary uses the multiplicative weights algorithm to assign
weights to groups (as per~\citet{freund96}) and the designer best-responds using the dynamic program
from \Cref{sec:dp} to solve Equation~\eqref{eq:noregret_optim} to choose an optimal set of products, noting that
\begin{align*}
 \E_{k \sim D(t)} \left[\regret_{G_k}(\cvec)\right]
& = \sum_{k \in [g]} D_k(t) \sum_{i \in G_k}  \frac{R_{\tau_i} (\cvec)}{\vert G_k \vert}
= \sum_{i=1}^n  R_{\tau_i} (\cvec)  \sum_{k \in [g]} \frac{D_k(t) \ind{i \in G_k}}{\vert G_k \vert}
= \regret_S \left(\cvec, \w (t) \right)
\end{align*}
where $w_i (t) \triangleq \sum_{k \in [g]} \frac{D_k(t)}{\vert G_k \vert} \ind{i \in G_k}$ denotes the weight assigned to agent $i$, at time step $t$.

\shortversion{
\begin{algorithm}\label{alg:game}
\begin{algorithmic}
\State \textbf{Input:} $p$ target number of products, consumers $S$ partitioned in groups $G_1,\ldots,G_g$, and $T$.
\State \textbf{Initialization:} The no-regret player picks the uniform distribution $D(1)$ over groups. 
\For{$t = 1, \ldots, T$}
	\State The designer sets $\cvec(t) = (c_1(t),\ldots,c_p(t)) \in C_p(S)$ so as to solve
    	\begin{align}\label{eq:noregret_optim}
    	\cvec(t) = \argmin_{\cvec \in C_p(S)} \E_{k \sim D(t)} \left[\regret_{G_k}(\cvec)\right].
    	\end{align}
	\State The adversary computes $u_k(t) = \regret_{G_k}(\cvec(t))/B$ for all $k \in [g]$ and sets $D(t+1)$ via multiplicative weight update with $\beta = \frac{1}{1 + \sqrt{2 \frac{\ln g}{T}}} \in (0,1)$, as follows:
    	\[
    	D_k(t+1) = \frac{D_k(t) \beta^{u_k(t)}}{\sum_{h=1}^g D_h(t) \beta^{u_h(t)}}~\forall k \in [g],
    	\]
\EndFor

\State \textbf{Output:} $\cC_T$: the uniform distribution over $\{ \cvec(t)\}_{t=1}^T$.
\caption{2-Player Dynamics for the {\em ex ante\/} Minmax Problem}\label{alg: dynamics}
\end{algorithmic}
\end{algorithm}
}

\longversion{
\begin{algorithm}
\begin{algorithmic}
\State \textbf{Input:} $p$ target number of products, consumers $S$ partitioned in groups $G_1,\ldots,G_g$, and $T$. 
\State \textbf{Initialization:} The no-regret player picks the uniform distribution $D(1) = \left(\frac{1}{g},\ldots,\frac{1}{g}\right) \in \Delta([g])$,
\For{$t = 1, \ldots, T$}
	\State The best-response player chooses $\cvec(t) = (c_1(t),\ldots,c_p(t)) \in C_p(S)$ so as to solve
    	\begin{align}\label{eq:noregret_optim}
    	\cvec(t) = \argmin_{\cvec \in C_p(S)} \E_{k \sim D(t)} \left[\regret_{G_k}(\cvec)\right].
    	\end{align} 
	\State The no-regret player observes $u_k(t) = \regret_{G_k}(\cvec(t))/B$ for all $k \in [g]$. The no-regret player sets $D(t+1)$ via multiplicative weight update with $\beta = \frac{1}{1 + \sqrt{2 \frac{\ln g}{T}}} \in (0,1)$, as follows:
    	\[
    	D_k(t+1) = \frac{D_k(t) \beta^{u_k(t)}}{\sum_{h=1}^g D_h(t) \beta^{u_h(t)}}~\forall k \in [g],
    	\]
\EndFor

\State \textbf{Output:} $\cC_T$: the uniform distribution over $\{ \cvec(t)\}_{t=1}^T$.
\caption{2-Player Dynamics for the Ex Ante Minmax Problem}\label{alg: dynamics}
\end{algorithmic}
\end{algorithm}

}

Theorem~\ref{thm:game} shows that the time-average of the strategy of the designer in Algorithm~\ref{alg: dynamics} is an approximate solution to minmax problem~\eqref{eq:fair-opt-rand}.\shortversion{
The proof can be found in the full version of the paper.
}

\begin{thm}\label{thm:game}
Suppose that for all $i \in [n]$, $r(\tau_i) \leq B$. Then for all $T > 0$, Algorithm~\ref{alg: dynamics} runs in time  $O(T n^2 p)$ and the output distribution $\cC_T$ satisfies
\begin{align}
\max_{k \in [g]} \E_{\cvec \sim \cC_T}\left[ \regret_{G_k}(\cvec)\right]
&\le \Rfairrand \left(S,p \right) +  B \left(\sqrt{\frac{2 \ln g}{T}} + \frac{\ln g}{T}\right).
\end{align}
\end{thm}

\longversion{
\begin{proof}
Note that the action space of the designer $C_p(S)$ and the action set of the adversary $\{G_k\}_{k=1}^g$ are both finite, so our zero-sum game can be written in normal form. Further, $u_k(t) \in [0,1]$, noting that the return of each agent is in $[0,B]$ --- so must be the average return of a group. Therefore, our minmax game fits the framework of~\citet{freund96}, and we have that
\[
\max_{k \in [g]} \E_{\cvec \sim \cC_T}\left[ \regret_{G_k}(\cvec)\right]
\leq \min_{\mathcal{C} \in \Delta(C_p(S))} \max_{\mathcal{G} \in \Delta([g])} \E_{k \sim \mathcal{G}, \, \cvec \sim \cC}  \left[ \regret_{G_k}(\cvec) \right]+  B \left(\sqrt{\frac{2 \ln g}{T}} + \frac{\ln g}{T}\right).
\]
The approximation statement is obtained by noting that for any distribution $\mathcal{G}$ over groups,
\[
\max_{\mathcal{G} \in \Delta([g])} \E_{k \sim \mathcal{G}, \, \cvec \sim \cC}  \left[ \regret_{G_k}(\cvec) \right] = \max_{1 \leq k \leq g} \E_{ \cvec \sim \cC}  \left[ \regret_{G_k}(\cvec) \right].
\]
With respect to running time, note that at every time step $t \in [T]$, the algorithm first solves
    	\[
    	\cvec(t) = \argmin_{\cvec \in C_p(S)} \E_{k \sim D(t)} \left[\regret_{G_k}(\cvec)\right].
    	\]
We note that this can be done in time $O(n^2 p)$ as per Section~\ref{sec:dp}, remembering that
\begin{align*}
 \E_{k \sim D(t)} \left[\regret_{G_k}(\cvec)\right]
= \regret_S \left(\cvec, \w (t) \right)
\end{align*}
i.e., Equation~\eqref{eq:noregret_optim} is a weighted regret minimization problem. Then, the algorithm computes $u_k(t)$ for each group $k$, which can be done in time $O(n)$ (by making a single pass through each customer $i$ and updating the corresponding $u_k$ for $i \in G_k$). Finally, computing the $g$ weights takes $O(g) \leq O(n)$ time. Therefore, each step takes time $O(n^2 p)$, and there are $T$ such steps, which concludes the proof.
\end{proof}
}

Importantly, Algorithm~\ref{alg: dynamics} outputs a \emph{distribution} over $p$-sets of products; Theorem~\ref{thm:game} shows that this distribution $\cC_T$ approximates the \emph{ex ante} minmax regret $\Rfairrand$. $\cC_T$ can be used to construct a \emph{deterministic} set of product with good regret guarantees: while each $p$-set in the support of $\cC_T$ may have high group regret, the union of all such $p$-sets must perform at least as well as $\cC_T$ and therefore meet benchmarks $\Rfairrand$ and $\Rfairdet$. However, this union may lead to an undesirable blow-up in the number of deployed products. 
The experiments in Section~\ref{sec:experiments} show how to avoid such a blow-up in practice.

\label{sec:nr}

\subsection{An {\em ex post} Minmax Fair Strategy for Few Groups}\label{sec:pure_minmax_few_groups}

In this section, we present a dynamic programming algorithm to find $p$ products that approximately optimize the {\em ex post\/} maximum regret across groups, as per Program~\eqref{eq:fair-opt-det}.\longversion{ Remember the optimization problem is given by
\begin{align}\tag{\ref{eq:fair-opt-det}}
\Rfairdet \left(S,p\right) = \min_{\cvec \, \in C_p (S)} \left\{ \max_{1 \le k \le g} \regret_{G_k} (\cvec) \right\}
\end{align}}
The algorithm aims to build a set containing all $g$-tuples of average regrets $(\regret_{G_1},\ldots,\regret_{G_g})$ that are simultaneously achievable for groups $G_1,\ldots,G_g$. However, doing so may be computationally infeasible, as there can be as many regret tuples as there are ways of choosing $p$ products among $n$ consumer thresholds, i.e. $\binom{n}{p}$ of them. Instead, we discretize the set of possible regret values for each group and build a set that only contains rounded regret tuples via recursion over the number of products $p$. The resulting dynamic program runs efficiently when the number of groups $g$ is a small constant and can guarantee an arbitrarily good additive approximation to the minmax regret. \shortversion{The algorithm is provided in the supplement.}\longversion{We provide the main guarantee of our dynamic program below and defer the full dynamic program and all technical details to Appendix~\ref{app:pure-minmax-extended}}

\begin{thm}
Fix any $\varepsilon > 0$. There exists a dynamic programming algorithm that, given a collection of consumers $S = \{ \tau_i \}_{i=1}^n$, groups $\{G_k\}_{k=1}^g$, and a target number of products $p$, finds a product vector $\cvec \in C_p(S)$ with $\max_{k \in [g] }\regret_{G_k}(\cvec) \leq \Rfairdet \left(S,p \right) + \varepsilon$ in time $O\left(n^2 p \left(\left\lceil \frac{Bp}{\varepsilon} \right\rceil + 1 \right)^g\right)$.
\end{thm}

This running time is efficient when $g$ is small, with a much better dependency in parameters $p$ and $n$ than the brute force approach that searches over all $\binom{n}{p}$ ways of picking $p$ products.

\newcommand \satset {\operatorname{SAT}}

\subsection{The Special Case of Interval Groups}\label{sec:interval_minmax}

We now consider the case in which each of the $g$ groups $G_k$ is defined by
an interval of risk thresholds. I.e., there are risk limits $b_1, \dots,
b_{g+1}$ so that $G_k$ contains all consumers with risk limit in $[b_k,
b_{k+1})$. Equivalently, every consumer in $G_k$ has risk limit strictly smaller than any member of $G_{k+1}$.

Our main algorithm in this section is for a decision version of the fair product
selection problem, in which we are given interval groups $G_1, \dots, G_g$, a
number of products $p$, a target regret $\kappa$, and the goal is to output a
collection of $p$ products such that every group has average regret at most
$\kappa$ if possible or else output \textsc{Infeasible}. We can convert any
algorithm for the decision problem into one that approximately solves the
minmax regret problem by performing binary search over the target regret
$\kappa$ to find the minimum feasible value. Finding an $\epsilon$-suboptimal
set of products requires $O(\log \frac{B}{\epsilon})$ runs of the decision
algorithm, where $B$ is a bound on the minmax regret.

Our decision algorithm processes the groups in order of increasing risk limit,
choosing products $\cvec^{(k)} \subset G_k$ when processing group $G_k$.
Given the maximum risk product $x = \max( \cup_{h=1}^{k-1} \cvec^{(h)})$
chosen for groups $G_1, \dots, G_{k-1}$, we choose $\cvec^{(k)} \subset
G_k$ to be a set of products of minimal size that achieves regret at most
$\kappa$ for group $G_k$ (together with the product $x$). Among all
smallest product sets satisfying this, we choose one with the product with the highest risk.
We call such a set of products \emph{efficient} for group $G_k$. We argue
inductively that the union $\cvec = \bigcup_{k=1}^g \cvec^{(k)}$ is
the smallest set of products that achieves regret at most $\kappa$ for all
groups. In particular, if $|\cvec| \leq p$ then we have a solution to the
decision problem, otherwise it is infeasible. We may also terminate the
algorithm early if at any point we have already chosen more than $p$ products.\shortversion{
We defer the technical details to the supplement and provide the guarantees of our algorithm below:
\begin{thm}
  There exists an algorithm that, given a collection of consumers divided into
  interval groups $S = \{G_k\}_{k=1}^g$ and a number of products $p$, outputs a
  set $\cvec$ of $p$ products satisfying $\max_{k \in [g]}
  \regret_{G_k}(\cvec) \leq \optregret(S, p) + \epsilon$ and runs in time
  $O(\log(\frac{B}{\epsilon}) \cdot p \sum_{k=1}^g |G_k|^2)$ where $B$ is a
  bound on the maximum regret of any group.
\end{thm}
}

\longversion{
Formally, for any set $S$ of consumer risk limits, number of products $p'$,
default product $x \leq \min(S)$, and target regret $\kappa$, define $\satset(S,
p', x, \kappa) = \{ \cvec \subset S \,:\, |\cvec| \leq p',
\regret_S(\cvec \cup \{x\}) \leq \kappa \}$ to be the (possibly empty)
collection of all sets of at most $p'$ products selected from $S$ for which the
average regret of the consumers in $S$ falls below $\kappa$ using the products
$\cvec$ together the default product $x$. For any collection of product
sets $A$, we say that the product set $\cvec \in A$ is efficient in $A$
whenever for all $\mathbf{d} \in A$ we have $|\cvec| \leq |\mathbf{d}|$ and
if $|\cvec| = |\mathbf{d}|$  then $\max(\cvec) \geq \max(\mathbf{d})$.
That is, among all product sets in $A$, $\cvec$ has the fewest possible
products and, among all such sets it has a maximal highest risk product. Each
iteration of our algorithm chooses an efficient product set from $\satset(G_k,
p', x, \kappa)$, where $p'$ is the number of products left to choose and $x$ is
the highest risk product chosen so far. Pseudocode is given in
\Cref{alg:intervalDecisionAlg}.

\begin{algorithm}
\textbf{Input:} Interval groups $G_1, \dots, G_g$, max products $p$, target regret $\kappa$
\begin{enumerate}[leftmargin=*, nosep]
\item Let $\cvec \leftarrow \emptyset$
\item For $k = 1, \dots, g$
\begin{enumerate}
\item If $\satset(G_k, p - |\cvec|, \max(\cvec), \kappa) = \emptyset$ output \textsc{Infeasible}
\item Otherwise, let $\cvec^{(k)}$ be an efficient set of products in $\satset(G_k, p - |\cvec|, \max(\cvec), \kappa)$.
\item Let $\cvec \leftarrow \cvec \cup \cvec^{(k)}$.
\end{enumerate}
\item Output $\cvec$.
\end{enumerate}
\caption{Fair Product Decision Algorithm} \label{alg:intervalDecisionAlg}
\end{algorithm}

\begin{lemma} \label{lem:intervalEfficiency}
  For any interval groups $G_1$, \dots, $G_k$, number of products $p$, and
  target regret $\kappa$, \Cref{alg:intervalDecisionAlg} will output a set of at
  most $p$ products for which every group has regret at most $\kappa$ if one
  exists, otherwise it outputs \textsc{Infeasible}.
\end{lemma}

\begin{proof}
See Appendix~\ref{app:interval}.
\end{proof}

It remains to provide an algorithm that finds an efficient set of products in
the set $\satset(S, p', x, \kappa)$ if one exists. The following Lemma shows
that we can use a slight modification of the dynamic programming algorithm from \Cref{thm:dp} to find such a set of products in $O(|S|^2 p')$ time if one exists, or output \textsc{Infeasible}.

\begin{lemma}
  There exists an algorithm for finding an efficient set of products in
  $\satset(S, p', x, \kappa)$ if one exists and outputs \textsc{Infeasible}
  otherwise. The running time of the algorithm is $O(|S|^2 p')$.
\end{lemma}
\begin{proof}
 A straightforward modification of the dynamic program described in \Cref{sec:dp} allows us to solve the problem of minimizing the regret of population $S$ when using $p'$ products, and the default option is given by a product with risk limit $x \leq \tau$ for all $\tau \in S$ (instead of $c_0$). The dynamic program runs in time $O(|S|^2 p')$, and tracks the optimal set of $p''$ products to serve the $z-1$ lowest risk consumers in $S$ while assuming the remaining consumers are offered product $\tau_z$, for all $z \leq |S|$ and $p'' < p'$. Assuming $\satset(S, p', x, \kappa)$ is non-empty, one of these $z$'s corresponds to the highest product in an efficient solution, and one of the values of $p''$ corresponds to the number of products used in said efficient solution. Therefore, the corresponding optimal choice of products is efficient: since it is optimal, the regret remains below $\kappa$. It then suffices to search over all values of $p'$ and $z$ after running the dynamic program, which takes an additional time at most $O(|S| p')$.
\end{proof}

Combined, the above results prove the following result:
\begin{thm}
  There exists an algorithm that, given a collection of consumers divided into
  interval groups $S = {G_k}_{k=1}^g$ and a number of products $p$, outputs a
  set $\cvec$ of $p$ products satisfying $\max_{k \in [g]}
  \regret_{G_k}(\cvec) \leq \optregret(S, p) + \epsilon$ and runs in time
  $O(\log(\frac{B}{\epsilon}) p \sum_{k=1}^g |G_k|^2)$, where $B$ is a bound on
  the maximum regret of any group.
\end{thm}
\begin{proof}
  Run binary search on the target regret $\kappa$ using
  \Cref{alg:intervalDecisionAlg} together with the dynamic program. Each run
  takes $O(p \sum_{k=1}^g |G_k|^2)$ time, and we need to do $O(\log
  \frac{B}{\epsilon})$ runs.
\end{proof}
}

\section{Generalization Guarantees} \label{sec:generalization}

\longversion{\subsection{Generalization for Regret Minimization Absent Fairness}}
\shortversion{\paragraph{Generalization for Regret Minimization Absent Fairness.}}
Suppose now that there is a distribution $\D$ over consumer risk thresholds. Our goal is to find a collection of $p$ products that minimizes the
\emph{expected} (with respect to $\D$) regret of a consumer when we only have access to $n$ risk limits sampled from $\D$. For any distribution $\D$ over consumer risk limits and any $p$, we define $\regret_{\D}(\mathbf{c}) = \expect_{\tau \sim \D}[\regret_{\tau}(\mathbf{c})]$,
which is the \emph{distributional} counterpart of $\regret_S(\mathbf{c})$.

In  Theorem~\ref{lem:min_regret_sample_complexity}, we provide a generalization guarantee that shows it is enough to optimize $\regret_S(\cvec)$ when $S$ is a sample of size $n \geq 2B^2 \varepsilon^{-2}\log \left(4/\delta \right)$ drawn $i.i.d.$ from $\D$.
\shortversion{
We defer the proof to the supplement.
\begin{thm}[Generalization Absent Fairness]\label{lem:min_regret_sample_complexity}
	For any $\varepsilon > 0$ and
	$\delta > 0$, and for any target number of products $p$, if $S = \{ \tau_i \}_{i=1}^n$ is drawn $i.i.d.$ from $\D$ provided that
	$n \geq 2B^2 \varepsilon^{-2}\log \left(4/\delta \right)$, then with probability at least $1-\delta$, we have
  $\sup_{\mathbf{c} \in \reals^p_{\geq 0}} \bigl|\regret_S(\mathbf{c}) - \regret_{\D}(\mathbf{c})\bigr| \leq \epsilon$.
  %
\end{thm}
}
\longversion{
\begin{thm}[Generalization Absent Fairness]\label{lem:min_regret_sample_complexity}
	Let $r : \reals_{\ge 0} \to \reals$ be any non-decreasing function with bounded range $B$
	and $\D$ be any distribution over $\reals_{\geq 0}$. For any $\varepsilon > 0$ and
	$\delta > 0$, and for any target number of products $p$, if $S = \{ \tau_i \}_{i=1}^n$ is drawn $i.i.d.$ from $\D$ provided that
	\[
	n \geq \frac{2B^2 \log \left(4/\delta \right)}{\varepsilon^2}
	\]
	then with probability at least $1-\delta$, we have
	\[
  \sup_{\mathbf{c} \in \reals^p_{\geq 0}} \bigl|\regret_S(\mathbf{c}) - \regret_{\D}(\mathbf{c})\bigr| \leq \epsilon.
  \]
\end{thm}

\begin{proof}
See Appendix~\ref{app:generalization}.
\end{proof}
}

\shortversion{
The number of samples needed to be able to approximately optimize  $\regret_{\D}(\mathbf{c})$ to an additive $\varepsilon$ factor has a standard square dependency in $1/\varepsilon$ but is independent of the number of products $p$. This result may come as a surprise, especially given that standard measures of the sample complexity of our function class \emph{do} depend on $p$. Indeed, in the supplement, we examine uniform
convergence guarantees over $\cF_p$ via \emph{Pollard's pseudo-dimension} ($\pdim$)
\citep{Pollard84:Pdim} and show
that the complexity of this class measured by $\pdim$
grows with the target number of products $p$: $\pdim (\cF_p) \ge p$. 
}

\longversion{
The number of samples needed to be able to approximately optimize  $\regret_{\D}(\mathbf{c})$ to an additive $\varepsilon$ factor has a standard square dependency in $1/\varepsilon$ but is independent of the number of products $p$. This result may come as a surprise, especially given that standard measures of the sample complexity of our function class \emph{do} depend on $p$. Indeed, we examine uniform convergence guarantees over the function class $\cF_p = \{ f_{\cvec} (\tau) \, | \, \cvec \in \reals_{\ge 0}^p \}$ where $f_\cvec (\tau) = \max_{c_j \le \tau} r(c_j)$\footnote{Note that uniform convergence over $\cF_p$ is equivalent to uniform convergence over the class of $R_{\tau}(\cvec)$ functions.} via \emph{Pollard's pseudo-dimension} ($\pdim$)
\citep{Pollard84:Pdim} and show
that the complexity of this class measured by Pollard's pseudo-dimension
grows with the target number of products $p$: $\pdim (\cF_p) \ge p$.

\begin{definition}[Pollard's Pseudo-dimension]
  A class $\cF$ of real-valued functions P-shatters a set of points $\tau_1,
  \dots, \tau_n$ if there exist a set of ``targets'' $\gamma_1, \dots, \gamma_n$
  such that for every subset $T \subset [n]$ of point indices, there exists a
  function, say $f_T \in \cF$ such that $f_T(\tau_i) \geq \gamma_i$ if and only
  if $i \in T$. In other words, all $2^n$ possible above/below patterns are
  achievable for the targets $\gamma_1, \dots, \gamma_n$. The pseudo-dimension
  of $\cF$, denoted by $\pdim(\cF)$, is the size of the largest set of points
  that it P-shatters.
\end{definition}

\begin{lemma}\label{lem:pdim}
	Let $r : \reals \to \reals$ be any function that is strictly increasing on
	some interval $[a, b] \subset \reals$\footnote{While not always true, most natural instances of the choice of assets are such that $r(\tau)$ is strictly increasing on some interval $[a,b]$.}. Then for any $p$, the corresponding
	class of functions $\cF_p$ has $\pdim(\cF_p) \geq p$.
\end{lemma}

\begin{proof}
See Appendix~\ref{app:generalization}.
\end{proof}

}

\longversion{\subsection{Generalization for Fairness Across Several Groups}}
\shortversion{\paragraph{Generalization for Fairness Across Several Groups.}}

In the presence of $g$ groups, we can think of $\D$ as being a mixture over $g$ distributions, say $\{ \D_k \}_{k=1}^g$, where $\D_k$ is the distribution of group $G_k$, for every $k$. We let the weight of this mixture on component $\D_k$ be $\pi_k$. Let $\pi_\text{min} = \min_{1 \le k \le g} \pi_k$.
\longversion{
In this framing, a sample $\tau \sim \D$ can be seen as first drawing $k \sim \pi$ and then $\tau \sim \D_k$. Note that we have sampling access to the distribution $\D$ and cannot directly sample from the components of the mixture: $\{ \D_k \}_{k=1}^g$.
}

\shortversion{
\begin{thm}[Generalization under Fairness]\label{lem:fair_sample_complexity}
	For any $\varepsilon > 0$ and $\delta > 0$, and for any target number of products $p$, if $S = \{ \tau_i \}_{i=1}^n$ consisting of $g$ groups is drawn $i.i.d.$ from $\D$ provided 	that
	$
	n \geq 2 \pi_\text{min}^{-1} \left( 4 B^2 \varepsilon^{-2} \log \left( 8 g / \delta \right) + \log \left(2g / \delta \right) \right)
	$,
	then with probability at least $1-\delta$, we have
	$
	\sup_{\cvec \in \reals_{\ge 0}^p, \, k \in [g]} \left| \regret_{G_k} (\cvec) - \regret_{\D_k} (\cvec) \right| \le \epsilon
	$.
\end{thm}
We extend \Cref{lem:fair_sample_complexity} to the {\em ex ante\/} case with mixtures of products  $\mathcal{C} \in \Delta(\reals^p_{\geq 0})$ in the supplement.
}
\longversion{
\begin{thm}[Generalization with Fairness]\label{lem:fair_sample_complexity}
	Let $r : \reals_{\ge 0} \to \reals$ be any non-decreasing function with bounded range $B$
	and $\D$ be any distribution over $\reals_{\geq 0}$. For any $\varepsilon > 0$ and
	$\delta > 0$, and for any target number of products $p$, if $S = \{ \tau_i \}_{i=1}^n$ consisting of $g$ groups $\{ G_k \}_{k=1}^g$ is drawn $i.i.d.$ from $\D$ provided that
	\[
	n \geq \frac{2}{\pi_\text{min}} \left( \frac{4 B^2 \log \left( 8 g / \delta \right) }{\varepsilon^2 } + \log \left(2g / \delta \right) \right)
	\]
	then with probability at least $1-\delta$, we have
	\[
	\sup_{\cC \in \Delta \left(\reals_{\ge 0}^p \right), \, k \in [g]} \left| \expect_{\cvec \sim \cC} \left[ \regret_{G_k} (\cvec) \right] - \expect_{\cvec \sim \cC} \left[ \regret_{\D_k} (\cvec) \right] \right| \le \epsilon.
	\]
\end{thm}

\begin{proof}
See Appendix~\ref{app:generalization}.
\end{proof}
}

\section{Experiments}\label{sec:experiments}

In this section, we present experiments that aim to complement our theoretical results. The code for these experiments can be found at \url{https://github.com/TravisBarryDick/FairConsumerFinance}.

\paragraph{Data.} The underlying data for our experiments consists of time 
series of daily closing returns for 50 publicly
traded U.S. equities over a 15-year period beginning in 2005 and ending
in 2020; the equities chosen were those with the highest liquidity during this period. 
From these time series,
we extracted the average daily returns and covariance matrix, which
we then annualized by the standard practice of multiplying returns and covariances by 252, the number of trading days in a calendar year.
The mean annualized return across the 50 stocks is 0.13, and all but one are positive due
to the long time period spanned by the data. The correlation between returns and risk (standard deviation of returns) is
0.29 and significant at $P = 0.04$.
The annualized returns and covariances are then the basis for the computation of optimal portfolios given a specified risk limit $\tau$ as per Section~\ref{sec:model}.\footnote{For ease of understanding, here we 
consider a restriction of the Markowitz portfolio model of~\cite{markowitz1952portfolio} and Equation~\eqref{eq:return-function} in which short sales are not allowed, i.e. the weight assigned to each asset must be non-negative.}

In Figure~\ref{fig:data}(a), we show a scatter plot of risk vs. returns for the 50 stocks.
For sufficiently small risk values,
the optimal portfolio has almost all of its weight in cash, since all of the equities have higher risk
and insufficient independence. At intermediate values of risk, the optimal portfolio concentrates its
weight on just the 7 stocks\footnote{These 7 stocks are Apple, Amazon, Gilead Sciences, Monster Beverage Corporation,
	Netflix, NVIDIA, and Ross Stores.}
highlighted in red
in Figure~\ref{fig:data}(a).
This figure also plots the optimal risk-return frontier, which generally
lies to the northwest (lower risk and higher return) of the stocks themselves, due to the optimization's exploitation of independence.
The black dot highlights the optimal return for risk tolerance 0.1, for which we show the optimal portfolio weights in Figure~\ref{fig:data}(b).
Note that once the risk tolerance reaches that of the single stock with highest return
(red dot lying on the optimal frontier, representing Netflix), the frontier
becomes flat, since at that point the portfolio puts all its weight on this stock.

\paragraph{Algorithms and Benchmarks.}

We consider both population and group regret and a number of algorithms:
the integer linear program (ILP) for optimizing group regret; 
an implementation of the no-regret dynamics (NR) for group regret
described in Algorithm \ref{alg: dynamics};
a ``sparsified'' NR (described below); 
the dynamic program (DP) of Section~\ref{sec:dp} for optimizing population regret; 
and a greedy heuristic,
which iteratively chooses the product that reduces population regret the most.\footnote{In Appendix~\ref{app:greedy},
we show that the average population return $f_S(\cvec) = \frac{1}{n} \sum_{i} \max_{c_j \leq \tau_i} r(c_j)$ is submodular, and thus the greedy algorithm, which has the advantage of
$O(np)$ running time compared to the $O(n^2 p)$ of the DP, also enjoys the standard approximate submodular performance
guarantees.} We will evaluate each of these algorithms on both types of regret (and provide further details in Appendix~\ref{app:exp}).

Note that Algorithm \ref{alg: dynamics} outputs a {\em distribution\/} over sets of $p$ products; a natural way of
extracting a fixed set of products is to take the union of the support (which we refer to as algorithm NR below), 
but in principle this could lead to far more than $p$ products.
We thus sparsify this union with
a heuristic that extracts only $p+s$ total products 
(here $s \geq 0$ is a number of additional ``slack'' products allowed)
by iteratively removing the higher
of the two products closest to each other until we have reduced to $p+s$ products. This
heuristic is motivated by the empirical observation that the NR dynamics often produce
clumps of products close together, and we will see that it performs well for small values of $s$.

\begin{figure}[h]
\centering
        \subcaptionbox{Risk vs. Returns: Scatterplot and Optimal Frontier.}{
        \includegraphics[scale=0.33]{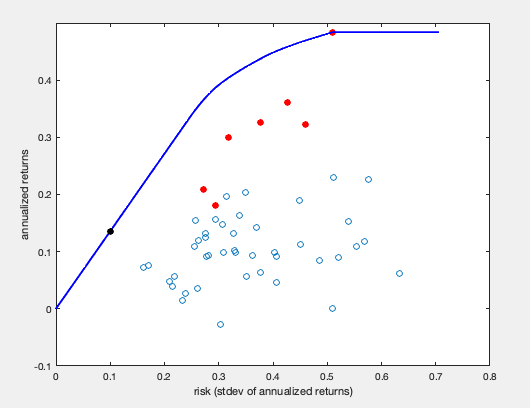}}
        \subcaptionbox{Optimal Portfolio Weights at Risk = 0.1.}{
        \includegraphics[scale=0.33]{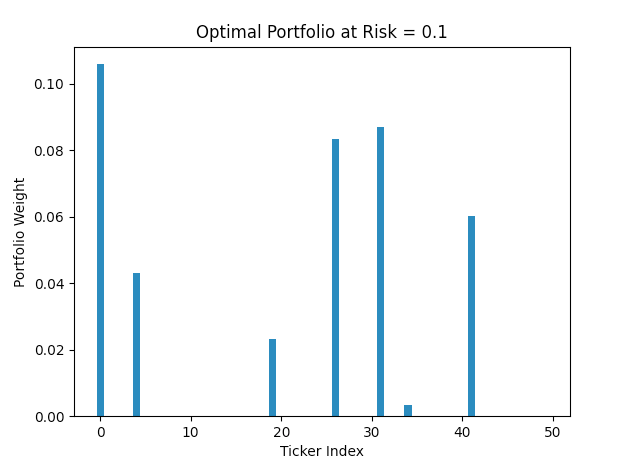}}
        \caption{Asset Risks and Returns, Optimal Frontier and Portfolio Weights}\label{fig:data}
\end{figure}

\paragraph{Experimental Design and Results.}

Our first experiment compares the empirical performance of the algorithms we have discussed on both population and group regret. We focus on a setting with $p=5$ desired products and 50 consumers drawn with uniform probability from 3 groups. Each group is defined by a Gaussian distribution of risk tolerances with $(\mu,\sigma)$ of $(0.02,0.002),(0.03,0.003)$ and $(0.04,0.004)$  (truncated at $0$ if necessary; as per the theory,
we add a cash option with zero risk and return). 
Thus the groups tend to be defined by risk levels, as in the interval groups case, but are noisy and therefore overlap in risk space.
The algorithms we compare include the NR algorithm run for $T=500$ steps; the NR sparsified to contain from 0 to 4 extra products; the ILP; the DP, and greedy. We compute population and group regret and average results over $100$ instances.

Figure \ref{fig:perf} displays the results\footnote{We provide additional design details and considerations in Appendix \ref{app:exp}.}. For both population and group regret, 
NR performs significantly better than ILP, DP, and greedy but also uses considerably larger numbers of products 
(between 8 and 19, with a median of 13).
The sparsified NR using $s = 0$ additional products results
in the highest regret 
but improves rapidly with slack allowed. 
By allowing $s = 2$ extra products, 
the sparsified NR achieves lower population regret than the DP with $p = 5$ products 
and lower group regret as the ILP with 5 products.  
While we made no attempt to optimize the integer program 
(besides using one of the faster solvers available), 
we found sparsified NR to be about two orders of magnitude 
faster than solving the ILP (0.3 vs. 14 seconds on average per instance).

\begin{figure}[h]\centering
	\subcaptionbox{Average Population Regret}{
		\includegraphics[scale=0.33]{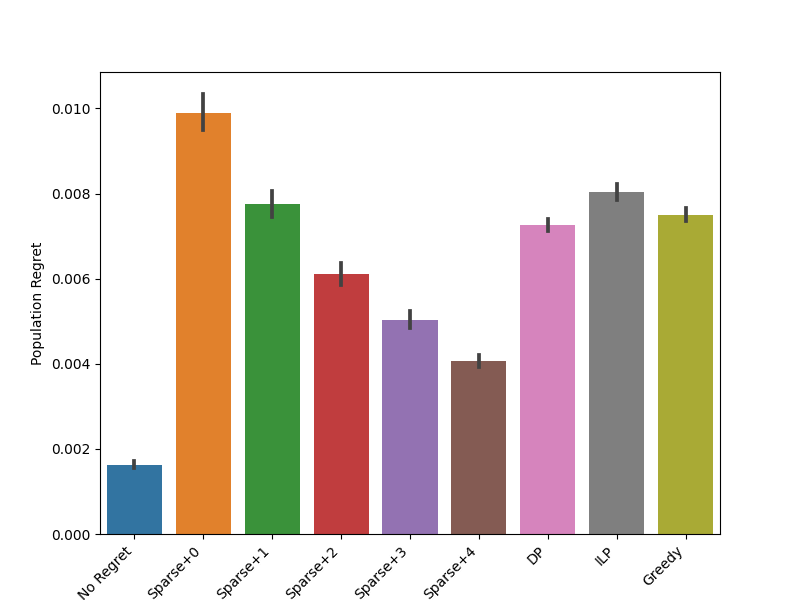}}
	\subcaptionbox{Average Group Regret}{
		\includegraphics[scale=0.33]{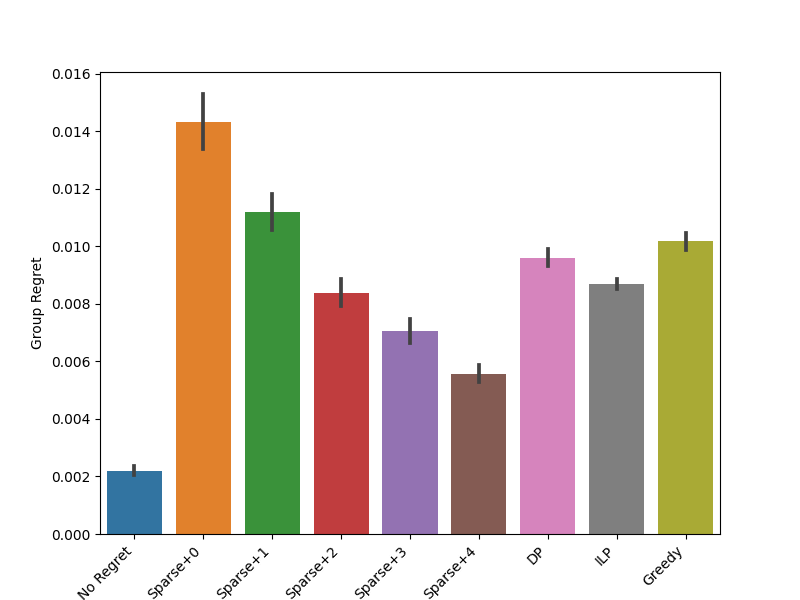}}
	\caption{\label{fig:perf} Algorithm Performance}
\end{figure}

Our second experiment explores generalization. 
We fix $p=5$ and use sparsified NR with a slack of $s = 4$ for a total of 9 products. 
We draw a test set of 5000 consumers from the same distribution described above. 
For sample sizes of $\{25,50,...,500\}$ consumers, we 
obtain product sets using sparsified NR and calculate the incurred regret using 
these products on the test set. We repeat this process 
100 times for each number of consumers and average them. 
This is plotted in Figure \ref{fig:gen}; we observe that measured both 
by population regret as well as by group regret, the test regret decreases as sample size increases. The decay rate is roughly $1/\sqrt{n}$, 
as suggested by theory, 
but our theoretical bound is significantly worse due to sub-optimal 
constants\footnote{The theoretical bound is roughly an order 
of magnitude higher than the experimental bound; we do not plot it as it makes the empirical errors difficult to see.}. 
Training regret increases with sample size because, for a fixed number of products, it is harder to satisfy a larger number of consumers.

\begin{figure}[!h]\centering
	\subcaptionbox{Population Regret}{
		\includegraphics[scale=0.33]{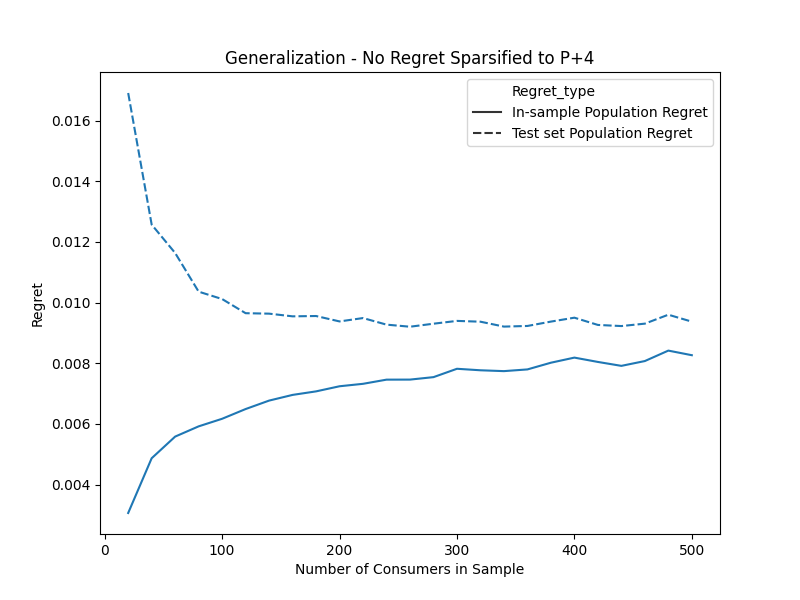}}
	\subcaptionbox{Group Regret}{
		\includegraphics[scale=0.33]{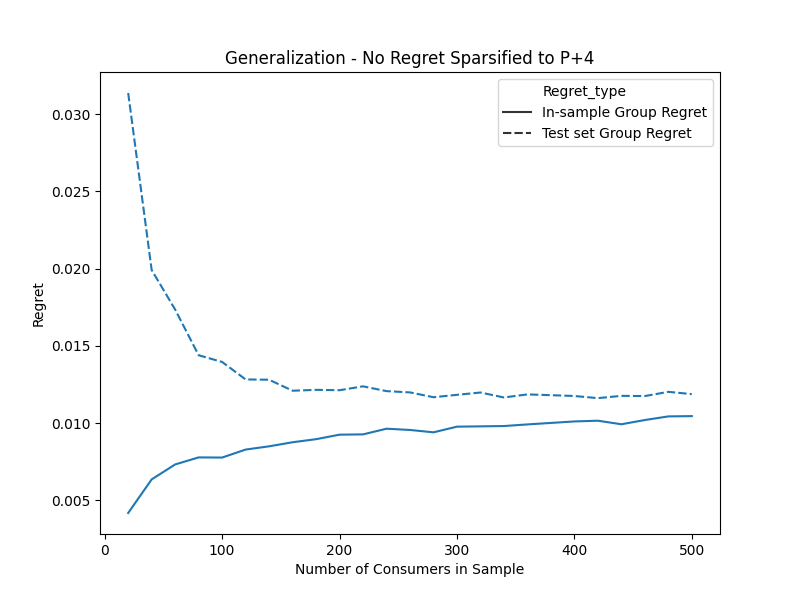}}
	\caption{\label{fig:gen} Generalization for No-Regret with $p=5$ Sparsified to $p + s = 9$ Products}
\end{figure}

\bibliographystyle{plainnat}
\bibliography{refs}
\newpage

\longversion{
\section*{Appendix}
\appendix

\section{Probabilistic Tools}
\begin{lemma}[Additive Chernoff-Hoeffding Bound]\label{lem:chernoff}
Let $X_1, \ldots, X_n$ be a sequence of $i.i.d.$ random variables with $a \le X_i \le b$ and $\E\left[ X_i \right] = \mu$ for all $i$. Let $B \triangleq b -a$. We have that for all $s > 0$,
$$
\prob \left[ \left\vert \frac{\sum_i X_i}{n} - \mu \right\vert \ge s \right] \le 2 \exp{\left( \frac{-2ns^2}{B^2} \right)}
$$
\end{lemma}

\begin{lemma}[Standard DKW-Inequality]\label{lem:DKW}
Let $\D$ be any distribution over $\reals$ and $X_1, \dots, X_n$ be an $i.i.d.$
  sample drawn from $\D$. Define
  $
  N(t) \triangleq \prob_{X \sim \D}(X \le t)
  $
  and
  $
  \hat N_n(t) \triangleq n^{-1} \sum_{i=1}^n \ind{X_i \le t}
  $
  to be the cumulative density functions of $\D$ and the drawn sample,
  respectively. We have that for all $s > 0$,
  $$
  \prob \left[ \sup_{t \in \reals} \left| \hat N_n(t) - N(t) \right| \ge s \right] \le 2 \exp{\left( -2ns^2 \right)}
  $$

\end{lemma}

The following Lemma is just a sanity check to make sure that we can apply the
DKW inequality to the \emph{strict} version of a cumulative density function.

\begin{lemma}[DKW-Inequality for Strict CDF]\label{lem:strictDKW}
  Let $\D$ be any distribution over $\reals$ and $X_1, \dots, X_n$ be an $i.i.d.$
  sample drawn from $\D$. Define
  $
  P(t) \triangleq \prob_{X \sim \D}(X < t)
  $
  and
  $
  \hat P_n(t) \triangleq n^{-1} \sum_{i=1}^n \ind{X_i < t}
  $
  to be the strict cumulative density functions of $\D$ and the drawn sample,
  respectively. We have that for all $s > 0$,
  $$
  \prob \left[ \sup_{t \in \reals} \left| \hat P_n(t) - P(t) \right| \ge s \right] \le 2 \exp{\left( -2ns^2 \right)}
  $$
\end{lemma}
\begin{proof}
  The key idea is to apply the standard DKW inequality to the (non-strict) CDF
  of the random variable $-X_i$. Towards that end,
  define $N(t) \triangleq \prob_{X \sim D}(-X \leq t)$ and $\hat N_n(t) \triangleq
  n^{-1} \sum_{i=1}^n \ind{-X_i \leq t}$. Then we have the following connection to $P$
  and $\hat P_n$:
  \begin{align*}
  \left|\hat P_n(t) - P(t) \right|
  &= \left |1 - \frac{1}{n} \sum_{i=1}^n \ind{X_i \ge t} - 1 + \prob_{X \sim D}(X \ge t) \right| \\
  &= \left | \frac{1}{n} \sum_{i=1}^n \ind{-X_i \le -t} - \prob_{X \sim D} (-X \le -t) \right| \\
  &= \left|\hat N_n(-t) - N(-t) \right|
  \end{align*}
  The proof is complete by the standard DKW inequality (see Lemma~\ref{lem:DKW}).
\end{proof}

\begin{lemma}[Multiplicative Chernoff Bound]\label{lem:mult-chernoff}
Let $X_1, \ldots, X_n$ be a sequence of $i.i.d.$ random variables with $0 \le X_i \le 1$ and $\E\left[ X_i \right] = \mu$ for all $i$. We have that for all $s > 0$,
$$
\prob \left[  \frac{\sum_i X_i}{n}  \le (1-s) \mu \right] \le \exp{\left( \frac{-n \mu s^2}{2} \right)}
$$
\end{lemma}

\section{A More General Regret Notion}\label{app:two-sided}

In this section we propose a family of regret functions parametrized by a positive real number that captures the regret notion we have been using throughout the paper as a special case. We will show how minor tweaks allow us to extend the dynamic program for whole population regret minization of Section~\ref{sec:dp}, as well as the no-regret dynamics for {\em ex ante\/} fair regret minimization of Section~\ref{sec:exante_fair}, to this more general notion of regret. For a given consumer with risk threshold $\tau$, and for any $\alpha \in (0, \infty]$, the regret of the consumer when assigned to a single product with risk threshold $c$ is defined as follows:
\begin{equation}\label{eq:new-regret-one-prod}
\regret_{\tau}^\alpha ( c ) = \begin{cases} r(\tau) - r(c) & c \le \tau 
\\
\alpha \left( c - \tau \right) & c > \tau
\end{cases}
\end{equation}
When offering more than one product, say $p$ products represented by $\cvec = (c_1, c_2, \ldots, c_p)$, the regret of a consumer with risk threshold $\tau$ is the best regret she can get using these products. Concretely,
\begin{equation}\label{eq:new-regret}
\regret_{\tau}^\alpha ( \cvec ) = \min_{1 \le i \le p} \regret_{\tau}^\alpha ( c_i )
\end{equation}
We note that our previous regret notion is recovered by setting $\alpha = \infty$. When $\alpha \neq \infty$, a consumer may be assigned to a product with higher risk threshold, and the consumer's regret is then measured by the difference between her desired risk and the product's risk, scaled by a factor of $\alpha$. We note that in practice, consumers may have different behavior as to whether they are willing to accept a higher risk product, i.e. different consumers may have different values of $\alpha$; in this section, we use a unique value of $\alpha$ for all consumers for simplicity of exposition and note that our insights generalize to different consumers having different $\alpha$'s. The regret of a set of consumers $S = \{\tau_i\}_{i=1}^n$ is defined as the average regret of consumers, as before. In other words,
\begin{equation}
\regret_{S}^\alpha ( \cvec ) = \frac{1}{n} \sum_{i=1}^n \regret_{\tau_i}^\alpha ( \cvec )
\end{equation}
We first show in Lemma~\ref{lem:one-prod} how we can find one single product to minimize the regret of a set of consumers represented by a set $S = \{ \tau_i \}_{i=1}^n$. Note this general formulation of regret will allow choosing products that do not necessarily fall onto the consumer risk thresholds. In Lemma~\ref{lem:one-prod} we will show that given access to the derivative function $r' (\tau)$, the problem $\min_{c \in \R_+} \regret_{S}^\alpha (c)$ can be reduced to an optimization problem that can be solved exactly in $O(n)$ time. Before that, we first observe the following property of function $r(\tau)$ (Equation~\eqref{eq:return-function}) which will be used in Lemma~\ref{lem:one-prod}:

\begin{clm}\label{clm:concave}
$r$ is a concave function.
\end{clm}
\begin{proof}[Proof of Claim~\ref{clm:concave}]
Let $X$ be the vector of random variables representing $m$ assets and recall $\mu$ is the mean of $X$ and $\Sigma$ is its covariance matrix. Fix $\tau_1, \tau_2 \ge 0$, and $\beta \in (0,1)$. For $i \in \{1,2\}$, let $\w_i$ be an optimal solution to the optimization problem for $r(\tau_i)$, i.e., $\w_i$ is such that $r(\tau_i) = \w_i^\top \mu$. We want to show that
\begin{equation}\label{eq:concav}
r \left(\beta \tau_1 + \left(1-\beta \right) \tau_2 \right) \ge \beta r(\tau_1) + (1-\beta) r(\tau_2) = \left( \beta \w_1 + \left(1-\beta \right) \w_2 \right)^\top \mu
\end{equation}
So all we need to show is that $\w \triangleq \beta \w_1 + (1-\beta) \w_2$ is feasible in the corresponding optimization problem for $r \left(\beta \tau_1 + \left(1-\beta \right) \tau_2 \right)$. Then, by definition, Equation~(\ref{eq:concav}) holds. We have that
\begin{align*}
\ones^\top \w = \beta ( \ones^\top \w_1 ) + (1-\beta) (\ones^\top \w_2) = 1
\end{align*}
and
\begin{align*}
\w^\top \Sigma \, \w &= \beta^2 \w_1^\top \Sigma \, \w_1 + (1-\beta)^2 \w_2^\top \Sigma \, \w_2 + 2 \beta (1-\beta) \w_1^\top \Sigma \, \w_2 \\
&\le \beta^2 \tau_1^2 + (1-\beta)^2 \tau_2^2 + 2 \beta (1-\beta) \cdot Cov \left( \w_1^\top X, \w_2^\top X\right) \\
&\le  \beta^2 \tau_1^2 + (1-\beta)^2 \tau_2^2 + 2 \beta (1-\beta) \cdot \sqrt{ Var(\w_1^\top X) Var(\w_2^\top X)} \\
& =\beta^2 \tau_1^2 + (1-\beta)^2 \tau_2^2 + 2 \beta (1-\beta) \cdot \sqrt{ (\w_1^\top \Sigma \, \w_1)(\w_2^\top \Sigma \, \w_2) } \\
& \le \beta^2 \tau_1^2 + (1-\beta)^2 \tau_2^2 + 2 \beta (1-\beta) \tau_1 \tau_2 \\
& = \left( \beta \tau_1 + (1-\beta) \tau_2 \right)^2
\end{align*}
where the second inequality follows from Cauchy-Schwarz inequality. Note also that $Cov \left( \w_1^\top X, \w_2^\top X\right) = \w_1^\top \Sigma \, \w_2$ and $Var(\w_i ^\top X) = \w_i^\top \Sigma \, \w_i$, for $i \in \{1,2\}$.
\end{proof}

\begin{lemma}\label{lem:one-prod}
Let $S = \{\tau_i\}_{i=1}^n$ where $\tau_1 \le \tau_2 \le \ldots \le \tau_n$. We have that $\regret_{S}^\alpha (c)$ is a convex function and
\begin{equation}\label{eq:one-prod}
\min_{c \in \R_{\ge 0}} \regret_{S}^\alpha (c) = \min \left\{ \min_{1 \le i \le n} \regret_{S}^\alpha ( \tau_i ), \min_{1 \le i \le n-1} \regret_{S}^\alpha ( c_i ) \cdot \indinf{ \tau_{i} < c_i < \tau_{i+1}} \right\}
\end{equation}
where for every $1 \le i \le n-1$, $c_i = (r')^{-1} \left(\frac{ \alpha i }{ n - i}\right)$ ($c_i = \infty$ if $(r')^{-1} \left(\frac{ \alpha i }{ n - i}\right)$ does not exist) and
\[
\indinf{ \tau_{i} < c_i < \tau_{i+1}} \triangleq \begin{cases} 1 & \tau_{i} < c_i < \tau_{i+1} \\ \infty & \text{otherwise} \end{cases}
\]
\end{lemma}

\begin{proof}[Proof of Lemma~\ref{lem:one-prod}]
First observe that Claim~\ref{clm:concave} implies for every $\tau$, $\regret_{\tau}^\alpha ( c )$ defined in Equation~\eqref{eq:new-regret-one-prod} is convex. Hence, $\regret_{S}^\alpha ( c )$ is convex because it is an average of convex functions. We have that for a single product $c$,
\[
\regret_{S}^\alpha ( c ) = \frac{1}{n} \sum_{j=1}^n \left\{ \left( r(\tau_j) - r(c) \right) \ind{c \le \tau_j} + \alpha \left( c - \tau_j \right) \ind{c > \tau_j} \right\}
\]
Note that $\regret_{S}^\alpha ( \tau_1 ) \le \regret_{S}^\alpha ( c )$ for every $c < \tau_1$ and $\regret_{S}^\alpha ( \tau_n ) \le \regret_{S}^\alpha ( c )$ for every $c > \tau_n$. We can therefore focus on the domain $[\tau_1, \tau_n]$ to find the minimum. The function $\regret_{S}^\alpha ( c )$ is differentiable everywhere except for the points given by consumers' risk thresholds: $S = \{\tau_i\}_{i=1}^n$. This justifies the first term appearing in the $\min \{ \cdot, \cdot \}$ term of Equation~\eqref{eq:one-prod}. For every $1 \le i \le n-1$, the function  $\regret_{S}^\alpha ( c )$ on domain $(\tau_i, \tau_{i+1})$ is differentiable and can be written as:
\[
\regret_{S}^\alpha ( c ) = \frac{1}{n} \left[ \alpha \sum_{j \le i} \left( c - \tau_j \right) + \sum_{j \ge i+1} \left( r(\tau_j) - r(c) \right) \right]
\]
The minimum of $\regret_{S}^\alpha ( c )$ on domain $(\tau_i, \tau_{i+1})$ is achieved on points $c$ where
\[
\frac{d}{dc} \regret_{S}^\alpha ( c ) = \frac{1}{n} \left[ \alpha i - r'(c) (n -i) \right] = 0 \quad \Longrightarrow \quad r' (c_i) = \frac{ \alpha i }{ n - i}
\]
We note that $ \regret_{S}^\alpha ( c )$ is a convex function by the first part of this Lemma implying that $c_i$ (if belongs to the domain $(\tau_i, \tau_{i+1})$) is a local \emph{minimum}. This justifies the second term appearing in the $\min \{ \cdot, \cdot \}$ term of Equation~\eqref{eq:one-prod} and completes the proof.
\end{proof}

\begin{rmk}
Given any set of weights $\w \in \reals_{\ge 0}^n$ over consumers, Lemma~\ref{lem:one-prod} can be easily extended to optimizing the weighted regret of a set of consumers given by:
\[
\regret_{S}^\alpha ( \cvec, \w ) = \sum_{i=1}^n w_i \regret_{\tau_i}^\alpha ( \cvec )
\]
In fact, for any $S$ and $\w$, we have that $\regret_{S}^\alpha ( c, \w )$ is a convex function and
\begin{equation}\label{eq:one-prod-weighted}
\min_{c \in \R_{\ge 0}} \regret_{S}^\alpha (c, \w) = \min \left\{ \min_{1 \le i \le n} \regret_{S}^\alpha ( \tau_i, \w ), \min_{1 \le i \le n-1} \regret_{S}^\alpha ( c_i, \w ) \cdot \indinf{ \tau_{i} < c_i < \tau_{i+1}} \right\}
\end{equation}
where for every $1 \le i \le n-1$, $c_i = (r')^{-1} \left(\frac{ \alpha \sum_{j\le i} w_j }{ \sum_{j \ge i+1} w_j }\right)$ ($c_i = \infty$ if $(r')^{-1} \left(\frac{ \alpha \sum_{j\le i} w_j }{ \sum_{j \ge i+1} w_j }\right)$ does not exist) and
\[
\indinf{ \tau_{i} < c_i < \tau_{i+1}} \triangleq \begin{cases} 1 & \tau_{i} < c_i < \tau_{i+1} \\ \infty & \text{otherwise} \end{cases}
\]

\end{rmk}
We now provide the idea behind a dynamic programming approach for choosing $p$ products that minimize the weighted regret of a population $S = \{\tau_i\}_{i=1}^n$. The approach relies on the simple observation that there exists an optimal solution such that if the consumers in set $S(p')$ are assigned to the $p'$-th product $c_{p'}$, $c_{p'} \in \argmin_{c \in \mathbb{R}^+}\regret_{S(p')}(c)$ (for a single product $c$)\footnote{On the one hand, for all $p'$ and given $S(p')$, picking $c_{p'}$ in such a manner provides a lower bound on the achievable average regret. On the other hand, $c_{p'}$ yields $S(p')$ as a set of consumers assigned to $c_{p'}$ in an optimal solution. Indeed, if consumers in $S(p')$ strictly preferred picking a different product, this could only be because they would get strictly better regret from doing so: by picking a different product, they would then decrease the population regret below our lower bound, which is a contradiction.}.

Therefore, to find an optimal choice of products, it suffices to i) correctly guess which subset $S(p')$ of consumers are assigned to the $p'$-th product, then ii) optimize the choice of product for $S(p')$, which can be done using Lemma~\ref{lem:one-prod}. Noting that $S(p')$ is an interval for all $p'$, $S(p')$ is entirely characterized by $z$, the first agent assigned to $c_{p'}$, and $n'$, the last agent assigned to $c_{p'}$. In turn, as in Section~\ref{sec:dp}, our dynamic program can be characterized by a recursive relationship of the form
\begin{align}\label{eq:two-sided_DP}
T(n', p') =
  \min_{z \in \{1, \dots, n'\}} \left(
	T(z-1, p'-1) + \min_{c \in \mathbb{R}^+} \sum_{i = z}^{n'} w_i \regret_{\tau_i}^\alpha ( c )
  \right),
  \end{align}
  where $T(n',p')$ represents the minimum weighted regret that can be achieved by providing $p'$ products to consumers $1$ to $n'$. Our dynamic program will implement this recursive relationship.
  
 The results of Section~\ref{sec:exante_fair} immediately extends to this more general regret notion, given that dynamic program~\eqref{eq:two-sided_DP} can be used as an optimization oracle for the problem solved by the best-response player in Algorithm~\ref{alg: dynamics}.

\section{Optimizing for Population vs. Least Well-Off Group: An Example}\label{app:pop_vs_group_optim}

In this section, we show that optimizing for population regret may lead to arbitrarily bad maximum group regret, and optimizing for maximum group regret may lead to arbitrarily bad population regret. To do so, we consider the following example: there is a set $S$ of $n$ consumers, divided into two groups $G_1$ and $G_2$. We let $|G_1| = 1$ and $|G_2| = n-1$ and assume the single consumer in group $G_1$ has risk threshold $\tau_1$, and the $n-1$ consumers in group $G_2$ all have the same risk threshold $\tau_2 > \tau_1$. We let $r_1 < r_2$ be the returns corresponding to risk thresholds $\tau_1,\tau_2$. Let $p=1$, i.e. the designer can pick only one product; either $\cvec = \tau_1$ or $\cvec = \tau_2$. 

When picking $\cvec = \tau_1$, we have that the average group and population regrets are given by: 
\begin{align*}
\regret_{G_1}(\tau_1) = 0,~\regret_{G_2}(\tau_1) = r_2 - r_1,~\regret_{S}(\tau_1) = \frac{(n-1) (r_2 - r_1)}{n}.
\end{align*}
When picking $\cvec = \tau_2$ instead, we have 
\begin{align*}
\regret_{G_1}(\tau_2) = r_1,~\regret_{G_2}(\tau_2) = 0,~\regret_{S}(\tau_2) = \frac{r_1}{n}.
\end{align*}
Suppose $0 < r_2 - r_1 < r_1$ and $n - 1 > \frac{r_1}{r_2 - r_1}$. Then, the optimal product to optimize for maximum group regret is $\cvec^{grp} = \tau_1$, and the optimal product to optimize for population regret is $\cvec^{pop} = \tau_2$. Then, we have that 
\begin{enumerate}
\item The ratio of population regret using $\cvec^{grp}$ over that of $\cvec^{pop}$ is given by:
\[
\frac{\regret_{S}(\cvec^{grp})}{\regret_{S}(\cvec^{pop})} = \frac{(n-1)(r_2 - r_1)/n}{r_1/n} = \frac{(n-1) (r_2 - r_1)}{r_1}.
\]
This ratio can be made arbitrarily large by letting $n \rightarrow +\infty$, at $\frac{r_2 - r_1}{r_1}$ constant.
\item The ratio of maximum group regret using $\cvec^{pop}$ over that of $\cvec^{grp}$ is given by:
\[
\frac{\max \left(\regret_{G_1}(\cvec^{pop}),~\regret_{G_2}(\cvec^{pop}) \right)}{\max \left(\regret_{G_1}(\cvec^{grp}),~\regret_{G_2}(\cvec^{grp}) \right)}
= \frac{r_1}{r_2 - r_1}.
\]
This ratio can be made arbitrarily large by letting $r_2 - r_1 \rightarrow 0$ at $r_1$ constant.
\end{enumerate}

\section{Approximate Population Regret Minimization via Greedy Algorithm}\label{app:greedy}
Recall that given $S = \{\tau_i\}_{i=1}^n$ and set $\cvec \subseteq S$ of products, the population regret of $S$ is given by 
\begin{align*}
 \regret_S (\cvec) = \frac{1}{n} \sum_{i=1}^n \regret_{\tau_i} (\cvec) =  \frac{1}{n} \sum_{i=1}^n \left( r(\tau_i) - f_{\tau_i} (\cvec) \right) = \frac{1}{n} \sum_{i=1}^n  r(\tau_i) - f_S (\cvec)
\end{align*}
where for any $\tau$, $f_{\tau} (\cvec) \triangleq \max_{c_j \le \tau} r (c_j)$, and 
\[
f_S : 2^S \to \reals_{\ge 0}, \quad f_S (\cvec) \triangleq \frac{1}{n} \sum_{i=1}^n f_{\tau_i} (\cvec).
\]
First, note that when no product is offered, consumers pick the cash option $c_0$ and get return $r(c_0) = 0$:
\begin{fact}[Centering]\label{fact:centered}
For any $S$, $f_S (\emptyset)  = 0$.
\end{fact}
Second, $f_S$ is immediately monotone non-decreasing, as consumers deviate to an additional product only when they get higher return from doing so: 
\begin{fact}[Monotonicity]\label{fact:monotone}
For any $S$, if $\cvec \subseteq \dvec$, then $f_S (\cvec) \le f_S (\dvec)$.
\end{fact}
Finally, $f_S$ is submodular: 
\begin{clm}[Submodularity]\label{clm:submod}
For any $S$, $f_S$ is submodular.
\end{clm}
\begin{proof}
We first show that for every $i$, $f_{\tau_i} (\cdot)$ is submodular: for any $\cvec, \dvec \subseteq S$, we have that
\[
f_{\tau_i} (\cvec \cup \dvec) + f_{\tau_i} (\cvec \cap \dvec) \le f_{\tau_i} (\cvec) + f_{\tau_i} (\dvec).
\]
If $\cvec = \emptyset$ or $\dvec = \emptyset$, the claim trivially holds. So assume $\cvec, \dvec \neq \emptyset$. Let $c$ be such that $f_{\tau_i} (\cvec \cup \dvec) = r(c)$. If $c = c_0 =0$, the claim holds because all four terms above will be zero. If $c \in \cvec$, the claim holds because $f_{\tau_i} (\cvec \cup \dvec) = f_{\tau_i} (\cvec)$ and $f_{\tau_i} (\dvec) \ge f_{\tau_i} (\cvec \cap \dvec)$ by Fact~\ref{fact:monotone}. The same argument holds when $c \in \dvec$ by symmetry. The proof is complete by noting that any simple average of submodular functions (in general, any linear combination with non-negative coefficients) is submodular.
\end{proof}

\begin{rmk}
Claim~\ref{clm:submod} extends to any weighted set of consumers: fix any set $S = \{\tau_i\}_{i=1}^n$ of consumers and any nonnegative weight vector $\w \in \reals_{\ge 0}^n$. Then the function $f_S : 2^S \to \reals$ given by $f_S (\cvec) = \sum_{i=1}^n w_i f_{\tau_i} (\cvec)$ is submodular.
\end{rmk}

We therefore have that for any $S$ and any target number of products $p$,
\[
\min_{\cvec \in C_p (S) }  \regret_S (\cvec) = \frac{1}{n} \sum_{i=1}^n r(\tau_i) - \max_{\cvec \subseteq S, \, |\cvec| = p} f_S (\cvec)
\]
where by Facts~\ref{fact:centered},~\ref{fact:monotone} and Claim~\ref{clm:submod}, the maximization problem in right-hand side is: \emph{maximization of a nonnegative monotone submodular function with a cardinality constraint}. Using the \emph{Greedy Algorithm} (that runs in $O(np)$ time), we get $p$ products represented by $\cvec^{grd}$ such that
\[
R_S (\cvec^{grd}) \le \frac{1}{n} \sum_{i=1}^n r(\tau_i) - \left(1 - e^{-1} \right) \cdot \max_{\cvec \subseteq S, \, |\cvec| = p} f_S (\cvec)
\]

\section{An {\em ex post} Minmax Fair Strategy for Few Groups, Extended}\label{app:pure-minmax-extended}
 
 Given bound $B$ on the maximum group regret and a step size of $\alpha$, we let
\[
N^{(\alpha)} \triangleq
\left\{
i \alpha:~ i  \in \left\{0,\ldots,\left\lceil \frac{B}{\alpha} \right\rceil \right\}
\right\}
\]
be a net of discretized regret values in $[0,B]$ with discretization size $\alpha$. Given any regret $\regret$, we let $\ceilstep{\alpha}(\regret)$ be the regret obtained by rounding $\regret$ up to the closest higher regret value in $N^{(\alpha)}$. I.e., $\ceilstep{\alpha}(\regret)$ is uniquely defined so as to satisfy $\regret \leq \ceilstep{\alpha}(\regret) < \regret + \alpha$ and $\ceilstep{\alpha}(\regret) \in N^{(\alpha)}$. Our dynamic program implements the following recursive relationship:
\begin{align}\label{eq: discrete_DP_regret}
\begin{split}
&\cF^{(\alpha)}(n',p') \triangleq
\\&\left\{
\left(
\ceilstep{\alpha} \left(\regret_{G_k} +\sum_{i = z}^{n'} \frac{\mathbbm{1} \left\{i \in G_k \right\}}{|G_k|}\regret_{\tau_i}(\tau_z) \right)
\right)_{k=1}^g
\text{s.t.}~z \leq n',~\left(\regret_{G_k}\right)_{k=1}^g \in \cF^{(\alpha)}(z-1,p'-1)
\right\}
\end{split}
\end{align}
where for all $n' \leq n$,
\begin{align}
\cF^{(\alpha)}(n',0) \triangleq
\left\{
\left(
\sum_{i = 1}^{n'} \frac{ \mathbbm{1} \left\{i \in G_k \right\} }{|G_k|} r(\tau_i)
\right)_{k=1}^g
 \right\}.
\end{align}
Note that $\cF^{(\alpha)}(n',0)$ contains a single regret tuple, whose $k$-th coordinate is the weighted regret of agents $G_k \cap \{1,\ldots,n'\}$ when using weight $1/|G_k|$ and offering no product.

Intuitively, $\cF^{(\alpha)}(n',p')$ keeps track of a rounded up version of the feasible tuples of weighted group regrets when using weight $1/|G_k|$ in group $G_k$ and when offering $p'$ products and considering the regret of consumers $1$ to $n'$ only. The corresponding set of products used to construct the regret tuples in $\cF^{(\alpha)}(n',p')$ can be kept in a hash table whose keys are the regret tuples in $\cF^{(\alpha)}(n',p')$; we denote such a hash table by $\solspace^{(\alpha)}(n,p)$. While there can be several $p'$-tuples of products that lead to the same rounded regret tuple in $\cF^{(\alpha)}(n',p')$, the dynamic program only stores one of them in the corresponding entry in the hash table at each time step. The program terminates after computing $\cF^{(\alpha)}(n,p),~\solspace^{(\alpha)}(n,p)$, and outputting the product vector in $\solspace^{(\alpha)}(n,p)$ corresponding to the regret tuple with smallest regret for the worst-off group in $\cF^{(\alpha)}(n,p)$. The sets $\cF^{(\alpha)}(n,p),~\solspace^{(\alpha)}(n,p)$ satisfy the following guarantee:

\begin{lemma}\label{eq: few_groups_regret_guarantee}
Fix any $\alpha > 0$. Let $(\regret_{G_1},\ldots,\regret_{G_g})$ be any regret tuple that can be achieved using $p$ products. There exist consumer indices $(z_1,\ldots,z_p) \in \solspace^{(\alpha)}(n,p)$  such that the corresponding regret tuple $(\regret_{G_1}^{(\alpha)},\ldots,\regret_{G_g}^{(\alpha)}) \in \cF^{(\alpha)}(n,p)$ satisfies
\begin{align*}
\regret_{G_k} (\cvec) \leq \regret_{G_k}^{(\alpha)} \leq \regret_{G_k} + p \alpha,~\forall k \in [g],
\end{align*}
where $\cvec = (\tau_{z_1},\ldots,\tau_{z_p})$.
\end{lemma}

We provide the proof of Lemma~\ref{eq: few_groups_regret_guarantee} in Appendix~\ref{app: pure_minmax_few_groups}. In particular, Lemma~\ref{eq: few_groups_regret_guarantee} implies that the dynamic program run with discretization parameter $\alpha$ approximately minimizes the maximum regret across groups (i.e., the optimal value of Program~\eqref{eq:fair-opt-det}) within an additive approximation factor of $p \alpha$. Letting $\alpha = \frac{\varepsilon}{p}$ yields an $\varepsilon$-approximation to the minmax regret.

The running time of our dynamic program is summarized below:
\begin{thm}
Fix any $\alpha > 0$. There exists a dynamic programming algorithm that, given a collection of consumers $S = \{ \tau_i \}_{i=1}^n$, groups $\{G_k\}_{k=1}^g$, and a target number of products $p$, computes $\cF^{(\alpha)}(n,p)$ and $\solspace^{(\alpha)}(n,p)$ in time $O\left(n^2 p \left(\left\lceil \frac{B}{\alpha} \right\rceil + 1 \right)^g\right)$.
\end{thm}

When the desired accuracy is $\varepsilon$, the dynamic program uses $\alpha = \frac{\varepsilon}{p}$ and has running time $O\left(n^2 p \left(\left\lceil \frac{Bp}{\varepsilon} \right\rceil + 1 \right)^g\right)$. 

\begin{proof}
Note that each set $\cF^{(\alpha)}(n',p')$ and $\solspace^{(\alpha)}(n',p')$ built by the dynamic program has size at most $ \left(\left\lceil \frac{B}{\alpha} \right\rceil + 1 \right)^g$, as $\cF^{(\alpha)}(n',p')$ only contains regret values in $N^{(\alpha)}$ by construction and $\solspace^{(\alpha)}(n',p')$ contains one product tuple per regret tuple in $\cF^{(\alpha)}(n',p')$.

Each sum used in the dynamic program can be computed in time $O(1)$, given that $\sum_{i = 1}^{n'} \frac{\mathbbm{1} \left\{i \in G_k \right\}}{|G_k|} r(\tau_i)$ and $\sum_{i=1}^{n'} \frac{\mathbbm{1} \left\{i \in G_k \right\}}{|G_k|}$ have been pre-computed for all $n' \in [n],k \in [g]$ and stored in a hash table. The pre-computation and storage of these partial sums can be done in $O(gn)$ time.

Now, each time step of the dynamic program corresponds to the $p'$-th product with $p' \leq p$. In each time step $p'$, we construct $O(n)$ sets $\cF^{(\alpha)}\left(n',p'\right)$, one for each value of $n'$. For each value of $n'$, the dynamic program searches over i) $z \in [n]$ and ii) at most $\left(\left\lceil \frac{B}{\alpha} \right\rceil + 1 \right)^g$ tuples of regret in $\cF^{(\alpha)}(z,p'-1)$. Therefore, building $\cF^{(\alpha)}(n,p)$ can be done in time $O\left(n^2 p \left(\left\lceil \frac{B}{\alpha} \right\rceil + 1 \right)^g \right)$

Finally, finding the tuple with the smallest maximum regret in $\cF^{(\alpha)}(n,p)$ requires searching over at most $\left(\left\lceil \frac{R}{\alpha} \right\rceil + 1 \right)^g$ product tuples in $\solspace^{(\alpha)}(n,p)$.

Therefore, running the dynamic program requires time $O\left(n^2 p \left(\left\lceil \frac{B}{\alpha} \right\rceil + 1 \right)^g\right)$. 
\end{proof}

\subsection*{Proof of Lemma~\ref{eq: few_groups_regret_guarantee}}\label{app: pure_minmax_few_groups}

We let $\cF(n',p')$ be the set of weighted regret tuples that are achievable using $p'$ thresholds when only agents $1$ to $n'$ are considered, and the regret of agents in group $G_k$ is weighted by $1/|G_k|$ -- importantly, this reweighting is independent of the choice of $n'$ and computes the average regret of a group as if all of its consumers were present. The set $\cF(n,p)$ contains all feasible tuples of average regret given groups $\{G_k\}_{k=1}^g$. Importantly, $\cF(n,p)$ is different from $\cF^{(\alpha)}(n,p)$: $\cF(n,p)$ contains all achievable tuples of regret, even those that do not belong to the net $N^{(\alpha)}$; in turn $\cF(n,p)$ may contain up to $\binom{n}{p}$ regret tuples. In comparison, $\cF^{(\alpha)}(n,p)$ is a smaller set of size at most $ \left(\left\lceil \frac{B}{\alpha} \right\rceil + 1 \right)^g$ that only contains regret tuples with values in the net $N^{(\alpha)}$. $\cF^{(\alpha)}(n,p)$ is built with the intent of approximating the true set of achievable regret tuples, $\cF(n,p)$.

The proof idea is to show that $\cF^{(\alpha)}(n,p)$ is a good approximation of $\cF(n,p)$, as intended. More precisely, we want to show that for every feasible regret tuple $\left(\regret_{G_k}\right)_{k=1}^g$ in $\cF(n,p)$, the discretized set $\cF^{(\alpha)}(n,p)$ contains a regret tuple $\left(\regret_{G_k}^{(\alpha)}\right)_{k=1}^g$ that is at most $p \alpha$ away from $\left(\regret_{G_k}\right)_{k=1}^g$; further, the corresponding product vector $\cvec^{(\alpha)} \in \solspace^{(\alpha)}(n,p)$ has true, \emph{unrounded} regret also within $p \alpha$ of $\left(\regret_{G_k}\right)_{k=1}^g$.

We start by showing that $\cF(n,p)$ obeys the following recursive relationship:
\begin{clm}\label{clm: recursion}
\begin{align}\label{eq: all_feasible_regrets}
\begin{split}
&\cF(n',p') =
\\&\left\{
\left(
\regret_{G_k} +\sum_{i = z}^{n'} \frac{\mathbbm{1} \left\{i \in G_k \right\}}{|G_k|}\regret_{\tau_i}(\tau_z)
\right)_{k=1}^g
\text{s.t.}~z \leq n',~\left(\regret_{G_k}\right)_{k=1}^g \in \cF(z-1,p'-1) 
\right\}
\end{split}
\end{align}
where 
\begin{align}
\cF(n',0) \triangleq \cF^{(\alpha)}(n',0) = 
\left\{  
\left(
\sum_{i = 1}^{n'} \frac{ \mathbbm{1} \left\{i \in G_k \right\} }{|G_k|} r(\tau_i) 
\right)_{k=1}^g
 \right\}.
\end{align} 
\end{clm}

\begin{proof}
When $p' = 0$, the result is immediate, noting that agents $1$ to $n'$ are assigned to the cash option $c_0$ with $0$ return and incur regret $r(\tau_i)$ each. The total regret in group $G_k$, considering agents $1$ to $n'$ and reweighting regret by $1/|G_k|$, is given by 
\[
\sum_{i = 1}^{n'} \frac{ \mathbbm{1} \left\{i \in G_k \right\} }{|G_k|} r(\tau_i).
\]

Now, take $p' > 0$. Fix the highest offered product to be $\tau_z$, corresponding to consumer $z$. First, agents $z,\ldots, n'$ are assigned to product $\tau_z$ corresponding to consumer $z$; the weighted regret incurred by these agents, limited to those in group $G_k$, is exactly 
\[
\sum_{i = z}^{n'}  \frac{\mathbbm{1} \left[i \in G_k\right]}{|G_k|} \left(r(\tau_i) - r(\tau_z)\right) = \sum_{i = z}^{n'}  \frac{\mathbbm{1} \left[i \in G_k\right]}{|G_k|} \regret_{\tau_i}(\tau_z) .
\]
The remaining agents are $1$ to $z-1$ and have $p'-1$ products available to them; hence, a regret tuple $(\regret_{G_1},\ldots,\regret_{G_g})$ can be feasibly incurred by these agents if and only if $(\regret_{G_1},\ldots,\regret_{G_g}) \in \cF(z-1,p'-1)$, by definition of $\cF(z-1,p'-1)$. To conclude the proof, it is enough to note that the total regret incurred by agents in group $G_k$ is the sum of the regrets of agents $\{1, \ldots, z-1\} \cap G_k$ and the regret of agents $\{z,\ldots,n'\} \cap G_k$.
\end{proof}

We now show that for any $\alpha > 0$, $\cF^{(\alpha)}(n,p)$ provides a $p\alpha$-additive approximation to the true set of possible regret tuples $\cF(n,p)$:
\begin{lemma}\label{lem:approx_cF}
For all $p \in \mathcal{N}$, for all $z \in [n]$, and for any regret tuple $(\regret_{G_1},\ldots,\regret_{G_g}) \in \cF(n,p)$, there exists a regret tuple $(\regret^{(\alpha)}_{G_1},\ldots,\regret^{(\alpha)}_{G_g}) \in \cF^{(\alpha)}(n,p)$ such that $\regret^{(\alpha)}_{G_k} \leq \regret_{G_k} + p \alpha$ for all $k\in [g]$.
\end{lemma}

\begin{proof}
The proof follows by induction on $p$. At step $p' \leq p$, the induction hypothesis states that for all $n'$, for any regret tuple $(\regret_{G_1},\ldots,\regret_{G_g}) \in \cF(n',p')$, there exists a regret tuple $(R^{(\alpha)}_{G_1},\ldots,R^{(\alpha)}_{G_g}) \in \cF^{(\alpha)}(n',p')$ such that $R^{(\alpha)}_{G_k} \leq \regret_{G_k} + p' \alpha$ for all $k \in [g]$.

First, when $p' = 0$, the induction hypothesis holds immediately: by definition, $\cF^{(\alpha)}(n',0)= \cF(n',0)$. Now, suppose the induction hypothesis holds for $p'-1$. Pick any regret tuple $(\regret_{G_1}(n',p'),\ldots,\regret_{G_g}(n',p')) \in \cF(n',p')$; we will show that the induction hypothesis holds for this tuple. First, note that there exists $z$ and $(\regret_{G_1}(z-1,p'-1),\ldots,\regret_{G_g}(z-1,p'-1)) \in \cF(z-1,p'-1)$ such that 
\[
\regret_{G_k}(n',p')= \regret_{G_k}(z-1,p'-1) + \sum_{i = z}^{n'}  \frac{\mathbbm{1} \left[i \in G_k\right]}{|G_k|} \regret_{\tau_i}(\tau_z) ~\forall k \in [g]
\]
by definition of $\cF(n',p')$. Further, by induction hypothesis, there exists a $g$-tuple of rounded regret
\\$(\regret_{G_1}^{(\alpha)}(z-1,p'-1),\ldots,\regret_{G_g}^{(\alpha)}(z-1,p'-1))$ in $\cF^{(\alpha)}(z-1,p'-1)$ such that
\begin{align*}
\regret_{G_k}^{(\alpha)}(z-1,p'-1) \leq \regret_{G_k}(z-1,p'-1) + (p'-1) \alpha~\forall k \in [g].
\end{align*}
Combining the above two equations, we get that 
\begin{align*}
&\regret_{G_k}^{(\alpha)}(z-1,p'-1) + \sum_{i = z}^{n'}  \frac{\mathbbm{1} \left[i \in G_k\right]}{|G_k|} \regret_{\tau_i}(\tau_z)
\\&\leq \regret_{G_k}(z-1,p'-1) + \sum_{i = z}^{n'}  \frac{\mathbbm{1} \left[i \in G_k\right]}{|G_k|} \regret_{\tau_i}(\tau_z) + (p'-1) \alpha
\\&= \regret_{G_k}(n',p') +(p'-1) \alpha~\forall k \in [g].
\end{align*}
Now, let $\regret_{G_k}^{(\alpha)}(n',p') = \ceilstep{\alpha}\left(\regret_{G_k}^{(\alpha)}(z-1,p'-1) + \sum_{i = z}^{n'}  \frac{\mathbbm{1} \left[i \in G_k\right]}{|G_k|} \regret_{\tau_i}(\tau_z) \right)$. First, $(\regret_{G_1}^{(\alpha)}(n',p'),\ldots,\regret_{G_g}^{(\alpha)}(n',p'))$ is in $\cF^{(\alpha)}(n',p')$ by definition. Second, 
\begin{align*}
\regret_{G_k}^{(\alpha)}(n',p') 
&= \ceilstep{\alpha}\left(\regret_{G_k}^{(\alpha)}(z-1,p'-1) + \sum_{i = z}^{n'}  \frac{\mathbbm{1} \left[i \in G_k\right]}{|G_k|}\regret_{\tau_i}(\tau_z) \right)
\\&\leq \regret_{G_k}^{(\alpha)}(z-1,p'-1) + \sum_{i = z}^{n'}  \frac{\mathbbm{1} \left[i \in G_k\right]}{|G_k|} \regret_{\tau_i}(\tau_z) + \alpha
\\&\leq \regret_{G_k}(n',p') + p' \alpha.
\end{align*}
This concludes the induction.
\end{proof}

To conclude the proof, let $(\regret_{G_1},\ldots,\regret_{G_g})$ be a tuple of regret that can be achieved using $p$ products. The tuple belongs to $\cF(n,p)$, by definition of $\cF(n,p)$. Therefore, by Lemma~\ref{lem:approx_cF}, there exists a regret tuple $(\regret_{G_1}^{(\alpha)},\ldots,\regret_{G_g}^{(\alpha)}) \in \cF^{(\alpha)}$ such that 
\[
\regret_{G_k}^{(\alpha)} \leq \regret_{G_k} + p \alpha~\forall k \in [g].
\]
Let $\cvec \triangleq \{c_1,\ldots,c_p\} \in \solspace^{(\alpha)}(n,p)$ be the product vector that was used to construct regret tuple $\left(\regret_{G_k}^{(\alpha)}\right)_{k=1}^g$. Since at each step $p'$, the dynamic program rounds regret tuples to higher values, a simple induction shows that
\[
\regret_{G_k} (\cvec) \leq \regret_{G_k}^{(\alpha)}~\forall k \in [g].
\]
Combining the two above equations, we get the result:
\[
\regret_{G_k} (\cvec) \leq \regret_{G_k}^{(\alpha)} \leq \regret_{G_k} + p \alpha~\forall k \in [g].
\]

\section{Additional Experiment Details}\label{app:exp}
All experiments were run on a consumer laptop without GPU (13-inch  Macbook Pro 2016 with 2 GHz Intel Core i5 and 8 GB of RAM), using Python 3.6.5 (and in particular, NumPy 1.17.1 and Pandas 1.0.3). Experiment 1 took about 16.7 minutes in total. Experiment 2 took about 3.8 hours in total. For reproducibility, we began each experiment with a fixed random seed (10). 

For the No-Regret experiments, we have two parameters to choose: the number of steps and the step size\footnote{We implement the Multiplicative Weight Update of Algorithm~\ref{alg: dynamics} in exponential form; i.e., we write $\beta = \exp(-\eta)$. The code takes the regret values $R_{G_k}(\cvec)$ as losses, takes $\eta/B$ as the step size, and updates the $k$-th weight in each round $t$ by a multiplicative factor of $\exp(-\eta R_{G_k}(\cvec(t))/B)$, as per Algorithm~\ref{alg: dynamics}. This allows us to easily translate the weights in $\log$ space so as to avoid numerical overflow issues.}. Note that theory suggests that step size should be a function of the number of steps desired (or vice versa), but common practice among related algorithms is to try many step sizes and run until convergence.  We adopt this approach, guided by theory. To do so, we calculate a lower bound on the theoretical instance-specific step size using the sum of rewards as a loose upper bound $B$ on the maximum possible group regret. We then consider applying the following multipliers to this lower bound on the step size: we take (1,10,100,1000,10000) as a small set of possible multipliers and examine convergence for each of them. In all experiments, we pick the step size multiplier after empirically observing that it both i) provides a good approximation to the optimal {\em ex ante\/} regret and ii) converges in few time steps; we note that our choices of multipliers enjoy good performance guarantees in practice, as evidenced by Figure~\ref{fig:perf}.

\section{Omitted Proofs}
\subsection{Proof of Lemma~\ref{lem:dpRelations}}\label{app:dynamic-program}

  The first property is immediate, because when $p' = 0$, we do not choose any
  products, and there is only one valid solution that assigns all consumers to
  the zero-risk cash product. For this solution, the weighted regret is simply
  the sum of the weighted returns for each consumer's bespoke portfolio.

  We now turn to proving the second property. For any consumer index $z \in
  [n']$, define $T(n', p', z)$ to be the optimal weighted regret for the first
  $n'$ consumers using $p'$ products subject to the constraint that the highest
  risk product has risk threshold set to $\tau_z$. That is,
  \[
    T(n', p', z)
    = \min_{\substack{
      \mathbf{c} = (c_1, \dots, c_{p'}) \subset S[n'] \\
      c_1, \dots, c_{p'} \leq \tau_z \\
      c_{p'} = \tau_z}
    }
    \regret_{S[n']}(\cvec,\mathbf{w}[n']).
  \]
  The products achieving weighted regret $T(n', p', z)$ must choose $c_{p'} =
  \tau_z$ and choose $c_1, \dots, c_{p'-1}$ to be optimal products for the
  consumers indexed $1, \dots, z-1$, who are not served by the product $c_{p'}$
  because it is too high risk. On the other hand, the weighted regret of
  consumers $z, \dots, n'$ when assigned to a product with risk limit $\tau_z$
  is given by $\sum_{i=z}^{n'} w_i \cdot \bigl( r(\tau_i) - r(\tau_z) \bigr)$.
  Together, this implies that
  \[
    T(n', p', z) = T(z-1, p'-1) + \sum_{i=z}^{n'} w_i \cdot \bigl
      (r(\tau_i) - r(\tau_z)
    \bigr).
  \]
  On the other hand, for any $n'$ and $p'$, we have that $T(n', p') = \min_{z
  \in [n']} T(n', p', z)$, because the optimal $p'$ products for the first $n'$
  consumers has a largest risk threshold equal to some consumer risk threshold.
  Combining these equalities gives
  \[
  T(n', p')
  = \min_{z \in [n']} T(n', p', z)
  = \min_{z \in [n']} T(z-1, p'-1) + \sum_{i=z}^{n'} w_i \cdot \bigl(
    r(\tau_i) - r(\tau_z)
  \bigr),
  \]
  as required.

\subsection{Proof of Theorem~\ref{thm:separation}}\label{app:separation}
\begin{proof}
Note there is a one-to-one relation (on some domain $[0,a]$ for risk thresholds) between any risk threshold $\tau$ and its corresponding return given by $r(\tau)$ by Equation~(\ref{eq:return-function}) --- we will therefore (only for simplicity of exposition) construct our instance by defining a set of returns instead of risk thresholds. Let $A \triangleq \{r,2r, \ldots, (p+1)r \}$ be the set of all possible returns in our instance for some constant $r>0$. We will take $r \equiv 1$ for simplicity but our proof extends to any $r >0$.

Our instance construction is simple. If $g \le p+1$, we partition $A$ into $g$ subsets of size at most $\lceil \frac{p+1}{g} \rceil$, and we let each group be defined as one of the partition elements. In this case, there will be one consumer for every return value in $A$. For e.g., if $p=4$ and $g=2$, we can define $G_1 = \{ 1, 2 , 3\}$ and $G_2 = \{ 4, 5\}$. If $g > p + 1$ (allowing consumers having the same return) we let each group be defined by a single return value in $A$. For e.g., if $p=2$ and $g=4$, we can define $G_1 = \{ 1 \}$, $G_2 = \{ 2 \}$, and $G_3 = G_4 = \{ 3 \}$. To formalize this construction, define $s \triangleq \min \{g,p+1\} $ and let $\{P_i\}_{i=1}^{s}$ be a $s$-sized partition of $R$ such that $\max_{i} \vert P_i \vert = \lceil \frac{p+1}{g} \rceil$. Instance $S$ of size $n = \max \{g, p+1\}$ is defined as follows.
\[
S =\{G_k\}_{k=1}^g \quad \text{where} \quad \forall \, k \in [g]: \quad G_k = \begin{cases}  P_k & k \le s \\  P_s & k > s \end{cases}
\]
Let $A_p = \{ B \subseteq A : \vert B \vert = p \}$ and observe that $\vert A_p \vert = p+1$. We have that
\[
\Rfairdet \left( S, p \right) \overset{(1)}{=} \min_{B \in A_p} \left\{ \max_{1 \le k \le g} \regret_{G_k} ( B ) \right\} \overset{(2)}{=} \frac{1}{\max_k \vert G_k \vert} \overset{(3)}{=} \frac{1}{\lceil \frac{p+1}{g} \rceil}
\]
where $(1)$ follows from the definition of $\Rfairdet \left( S, p \right)$ in this specific instance that all consumer returns are specified by the set $A$ of size $p+1$. $(2)$ follows from the fact that for any set of products $B \in A_p$, all groups $G_k$ that have a consumer with return $A \setminus B$ will incur an average regret of $1/|G_k|$. $(3)$ holds because $\max_k \vert G_k \vert = \max_{i} \vert P_i \vert = \lceil \frac{p+1}{g} \rceil$.
Next, by looking at the uniform distribution over $A_p$,
\[
\Rfairrand \left( S, p \right) \overset{(1)}{\le} \max_{1 \le k \le g} \frac{1}{p+1} \sum_{B \in A_p} \regret_{G_k} ( B )  \overset{(2)}{=}  \max_{1 \le k \le g} \frac{1}{p+1} \sum_{r \in G_k} \frac{1}{\vert G_k \vert} = \frac{1}{p+1}
\]
where $(1)$ follows from the definition of $\Rfairrand \left( S, p \right)$. $(2)$ follows from the fact that for every group $k$ and every $r \in G_k$, there is one (and only one) set of products, namely $B = A \setminus \{r\}$, that makes $G_k$ incur a regret of $1/|G_k|$.
We therefore have that
\[
\frac{\Rfairrand \left( S, p \right)}{\Rfairdet \left( S, p \right)} \le \frac{1}{p+1} \left\lceil \frac{p+1}{g} \right\rceil.
\]
\end{proof}

\subsection{Proof of Lemma~\ref{lem:intervalEfficiency}}\label{app:interval}

To ease notation, for any product sets $\cvec^{(1)}, \dots, \cvec^{(g)}$, we let
$\cvec^{(h:k)} = \bigcup_{\ell = h}^k \cvec^{(\ell)}$ denote the union of the
product sets with indices in $\{h, \dots, k\}$.

Suppose our feasibility problem has a solution. It is enough to show that there
exists product sets $\cvec^{(1)}, \dots, \cvec^{(g)}$ such that $\cvec^{(1:g)}$
is a feasible set of at most $p$ products and $\cvec^{(k)} \subset G_k$ is
efficient in $\satset(G_k, p - |\cvec^{(1:k-1)}|, \max(\cvec^{(1:k-1)}),
\kappa)$ for all $k \in [g]$. Note that by the definition of efficiency and a
straightforward induction, any product sets $\cvec^{(1)}, \dots \cvec^{(g)}$ and
alternative product sets $\cvec^{\prime(1)}, \dots, \cvec^{\prime(g)}$
satisfying these properties must have $|\cvec^{\prime (k)}| = |\cvec^{(k)}|$ and
$\max(\cvec^{\prime(k)}) = \max(\cvec^{(k)})$ for all $k \in [g]$, hence
\begin{align*}
\satset_k
&\triangleq \satset(G_k, p - |\cvec^{(1:k-1)}|, \max(\cvec^{(1:k-1)}), \kappa)
\\&= \satset(G_k, p - |\cvec^{\prime (1:k-1)}|, \max(\cvec^{\prime(1:k-1)}), \kappa)
\end{align*}
does not depend on the specific choice of $\cvec^{(1)}, \dots, \cvec^{(g)}$ that
satisfies the above assumptions. Since the algorithm may only output
\textsc{Infeasible} if $\satset_k$ is empty for some $k$, it can only do so when
no $\cvec^{(1)}, \dots, \cvec^{(g)}$ satisfying the above assumptions exists.
When such product sets exist, the algorithm outputs one, and $\cvec^{(1:g)}$ is
guaranteed to use at most $p$ products and have regret at most $\kappa$ in each
group by definition of $\satset(G_k, p - |\cvec^{(1:k-1)}|,
\max(\cvec^{(1:k-1)}), \kappa)$.

Assuming our problem is feasible with $p$ products, we show the following
induction hypothesis: for all $k \leq g$ there exist product sets $\cvec^{(1)},
\ldots, \cvec^{(k)}$ such that $\cvec^{(j+1)}$ is efficient in $\satset(G_j, p -
|\cvec^{(1:j)}|, \max(\cvec^{(1:j)}), \kappa)$ for all $j \leq k-1$, that can be
completed in a set products $\cvec^{(1:k)} \cup \dvec^{(k+1:g)}$ that has size
at most $p$ and guarantees regret at most $\kappa$ in each group, where
$\dvec^{(j)} \subset G_j$.

First, the induction hypothesis immediately hold for $k = 0$: since the problem
is feasible, there exists a product vector $\dvec$ that uses at most $p$
products and guarantees group regret of at most $\kappa$. Now, suppose the
induction hypothesis holds for $k$; we will show it holds for $k+1$. Take
$\cvec^{(1:k)}\cup \dvec^{(k+1:g)}$ that has size at most $p$ and guarantees
regret at most $\kappa$ in each group, such that $\cvec^{(j+1)}$ is efficient in
$\satset(G_j, p - |\cvec^{(1:j)}|, \max(\cvec^{(1:j)}), \kappa)$ for all $j \leq
k-1$. Take $\cvec^{(k+1)}$ to be efficient in $\satset(G_j, p - |\cvec^{(1:k)}|,
\max(\cvec^{(1:k)}), \kappa)$; we have two cases:
\begin{enumerate}
\item $|\cvec^{(k+1)}| = |\dvec^{(k+1)}|$ and $\max(\cvec^{(k+1)}) \geq
\max(\dvec^{(k+1)})$. Let us consider product set $\cvec^{(1:k+1)} \cup \dvec^{(k+2:g)}$. Then, $|\cvec^{(1:k)} \cup \dvec^{(k+1:g)}| =
|\cvec^{(1:k+1)} \cup \dvec^{(k+2:g)}| \leq p$. Further, no group has regret
over $\kappa$. Indeed, compared to when using $\cvec^{(1:k)} \cup
\dvec^{(k+1:g)}$, we have that: i) the regret of groups $G_1$ to $G_k$ is
unaffected as they only use products $\cvec^{(1:k)}$; ii) the regret of group
$G_{k+1}$ stays below $\kappa$ by satisfiability of $\cvec^{(k+1)}$; iii) the
regret of groups $G_{k+2},\ldots,G_{g}$ cannot increase because agents who used
product $\max(\dvec^{(k+1)})$ get weakly lower regret from using
$\max(\cvec^{(k+1)}) \geq \max(\dvec^{(k+1)})$, and the remaining agents can
keep using the same products in $\dvec^{(k+2:g)}$. Since the regret of all
groups remains below $\kappa$ under products $\cvec^{(1:k+1)} \cup
\dvec^{(k+2:g)}$ by the induction hypothesis, this remains true when using products
$\cvec^{(1:k+1)} \cup \dvec^{(k+2:g)}$.
\item $|\cvec^{(k+1)}| < |\dvec^{(k+1)}|$. In that case, let
$\tilde{\dvec}^{(k+2)} = \dvec^{(k+2)} \cup \{b_{k+2}\}$ (where $b_{k+2}$ is the
smallest threshold in group $G_{k+2}$). Consider product offering
$\cvec^{(1:k+1)} \cup \tilde{\dvec}^{(k+2)} \cup \dvec^{(k+3:g)}$. First, we
note that
\begin{align*}
|\cvec^{(1:k+1)} \cup \tilde{\dvec}^{(k+2)} \cup \dvec^{(k+3:g)}|
&\leq | \cvec^{(1:k)} \cup \dvec^{(k+1)} \cup \tilde{\dvec}^{(k+2)} \cup \dvec^{(k+3:g)}| - 1
\\&=  | \cvec^{(1:k)} \cup \dvec^{(k+1:g)}|
\\&\leq p.
\end{align*}
Second, note that the regrets of all groups remain under $\kappa$. Indeed,
compared to when offering products $\cvec^{(1:k)} \cup \dvec^{(k+1:g)}$, we
have: i) the regrets of groups $G_1,\ldots,G_k$ stay the same, as before; ii)
the regret of group $G_{k+1}$ stays below $\kappa$ by satisfiability of
$\cvec^{(k+1)}$, as before; iii)  the regret of groups $G_{k+2},\ldots,G_g$ can
also only decrease, because all agents who were assigned to $\max(\dvec^{(k+1)})
\leq b_{k+2}$ are now assigned to the higher return product $b_{k+2}$, while the
remaining agents stay assigned to the same product in $\dvec^{(k+2:g)}$.
\end{enumerate}
This concludes the proof.

\subsection{Proofs of Generalization Theorems}\label{app:generalization}

\begin{proof}[Proof of Theorem~\ref{lem:min_regret_sample_complexity}]
	Let $f_\cvec (\tau) \triangleq \max_{c_j \le \tau} r
  	(c_j)$ and observe that using this notation, for any $\tau$, $\regret_\tau
  	(\cvec) = r(\tau) - f_\cvec (\tau)$. To prove the claim of the theorem, first note that,
	\begin{equation}\label{eq:triangle}
	\sup_{\cvec \in \reals_{\ge 0}^p} \left| \regret_S (\cvec) - \regret_\D (\cvec) \right| \le \left|  \expect_{\tau \sim S} \left[ r (\tau) \right] - \expect_{\tau \sim \D} \left[ r (\tau) \right] \right| +\sup_{\cvec \in \reals_{\ge 0}^p} \left| \expect_{\tau \sim S}[f_\cvec(\tau)] -  \expect_{\tau \sim D}[f_\cvec(\tau)] \right|
	\end{equation}
	where ``$\tau \sim S$" means sampling $\tau$ from the uniform distribution over $S$. We have that by an application of additive Chernoff-Hoeffding bound (see Lemma~\ref{lem:chernoff}), with probability at least $1-\delta/2$, given the assumption on sample size $n$,
	\begin{equation}\label{eq:tri-1}
	\left|  \expect_{\tau \sim S} \left[ r (\tau) \right] - \expect_{\tau \sim \D} \left[ r (\tau) \right] \right| \le \sqrt{\frac{B^2 \log \left(4/\delta\right)}{2n}} \le \frac{\epsilon}{2}
	\end{equation}
	Now let us focus on the second term appearing in Equation~\eqref{eq:triangle}. Consider any vector of products $\cvec = (c_1, \ldots, c_p) \in \reals_{\ge 0} ^p$ where, without loss of generality, we assume $c_1 \le \dots \le c_p$. Letting $c_0 = 0$ and $c_{p+1} = \infty$, these products partition $\reals_{\geq 0}$
	into $p$ intervals $[c_j, c_{j+1})$ for $j \in [p]$ such that $f_\cvec(\tau) =
	r(c_j)$ for all $\tau \in [c_j, c_{j+1})$. We can rewrite $f_\cvec(\tau)$ as a
	telescoping sum that adds a term $r(c_j) -
	r(c_{j-1})$ for each interval $j$ up to and including the interval containing
	$\tau$:
	\begin{align*}
	f_\cvec (\tau)
	&= \sum_{j=0}^p r(c_j) \ind{c_j \leq \tau < c_{j+1}}
	\\&=  \sum_{j=0}^p r(c_j) \left(\ind{c_j \leq \tau} - \ind{c_{j+1} \leq \tau}\right)
	\\&= \sum_{j=0}^p r(c_j) \ind{c_j \leq \tau} - \sum_{j=0}^p r(c_j)\ind{c_{j+1} \leq \tau}
	\\&= \sum_{j=0}^p r(c_j) \ind{c_j \leq \tau} - \sum_{j=1}^{p+1} r(c_{j-1})\ind{c_{j} \leq \tau}
	\\&= r(c_0) \ind{c_0 \leq \tau}  - r(c_p) \ind{c_{p+1} \leq \tau} +  \sum_{j=1}^p \ind{\tau \geq c_j} \bigl(r(c_j) - r(c_{j-1})\bigr)
	\\&= \sum_{j=1}^p \ind{\tau \geq c_j} \bigl(r(c_j) - r(c_{j-1})\bigr)
	\end{align*}
remembering for the last equality that $c_0$ is the risk-free cash option and has return $r(c_0) = 0$ and that $c_{k+1} = +\infty$ is the final dummy product that corresponds to having an infinite risk and always satisfies $\tau < c_{p+1}$.
	Taking expectations of this expression converts the indicator into the
	complementary CDF (CCDF), and therefore we have
	\begin{align}\label{eq:tri-2}
	\begin{split}
	\sup_{\cvec \in \reals_{\ge 0}^p} \left|
	  \expect_{\tau \sim S}[f_\cvec (\tau)] - \expect_{\tau \sim \D}[f_\cvec(\tau)]
	\right|
	&=
	\sup_{\cvec \in  \reals_{\ge 0}^p} \left|
	  \sum_{j=1}^p \left( \prob_{\tau \sim S}(\tau \geq c_j) - \prob_{\tau \sim \D}(\tau \geq c_j) \right)
		\cdot \bigl(r(c_j) - r(c_{j-1})\bigr)
	\right| \\
	&\leq
	  \sup_{\cvec \in  \reals_{\ge 0}^p} \sum_{j=1}^p \left| \prob_{\tau \sim S}(\tau \geq c_j) - \prob_{\tau \sim \D}(\tau \geq c_j) \right|
		\cdot \bigl(r(c_j) - r(c_{j-1})\bigr) \\
	&\leq
	  \sup_{t \in \reals} \left|\prob_{\tau \sim S}(\tau \geq t) - \prob_{\tau \sim D}(\tau \geq t)\right| \cdot \sum_{j=1}^p \bigl(r(c_j) - r(c_{j-1})\bigr) \\
	&\leq
	  \sqrt{\frac{\log \left(4/\delta \right)}{2n}} \cdot \left( r(c_p) - r(c_1) \right) \\
	  & \leq \sqrt{\frac{B^2 \log \left(4/\delta \right)}{2n}} \\
	  &\le \frac{\epsilon}{2}
	\end{split}
	\end{align}
	where the first inequality follows from the triangle inequality and the fact that $r$ is non-decreasing:
	$r(c_j) \geq r(c_{j-1})$ for $j = 1, \dots, p$.
	The third inequality holds with probability $1-\delta/2$ and follows from the Dvoretzky-Kiefer-Wolfowitz inequality (see Lemma~\ref{lem:strictDKW}). The last inequality follows by the assumption on sample size $n$. Combining Equations~\eqref{eq:triangle} and \eqref{eq:tri-1} and \eqref{eq:tri-2}, completes the proof of the theorem.
\end{proof}

\begin{proof}[Proof of Theorem~\ref{lem:fair_sample_complexity}]
Let $S = \{\tau_i\}_{i=1}^n$ be a set of consumers of size $n$ drawn from the
distribution $\D$ that is partitioned into $g$ groups: $\{G_k\}_{k=1}^g$. Let
$n_k = \vert G_k \vert$ denote the size of group $k$. It follows from the
uniform convergence of Theorem~\ref{lem:min_regret_sample_complexity} that for
any group $k$, so long as $n_k \geq 2B^2 \epsilon^{-2} \log \left(4/\delta'
\right)$, with probability $1-\delta'$, $\sup_{\cvec \in \reals_{\ge 0}^p}
\left| \regret_{G_k} (\cvec) - \regret_{\D_k} (\cvec) \right| \le \epsilon$.
This implies by a union bound that with probability at least $1-\delta/2$, so
long as $n_k \geq 2B^2 \epsilon^{-2} \log \left(8g/\delta \right)$ for all $k$,
\[
\sup_{\cvec \in \reals_{\ge 0}^p, \, k \in [g]} \left| \regret_{G_k} (\cvec) - \regret_{\D_k} (\cvec) \right| \le \epsilon
\]
But note for any group $k$, $n_k \sim \text{Bin}(n,\pi_k)$ where
$\text{Bin}(m,q)$ denotes a binomial random variable with $m$ trials and success
probability $q$. By an application of the Multiplicative Chernoff bound (see
Lemma~\ref{lem:mult-chernoff}), as well as a union bound, we have that with
probability $1-\delta/2$, for any group $k$,
\[
n_k \ge n\pi_k - \sqrt{2 n \pi_k \log \left( 2g / \delta \right)} \ge 2B^2 \epsilon^{-2} \log \left(8g/\delta \right)
\]
where the second inequality follows by the assumption on $n$ in the theorem statement. We can therefore conclude by another union bound that, with probability at least $1-\delta$,
$$
\sup_{\cvec \in \reals_{\ge 0}^p, \, k \in [g]} \left| \regret_{G_k} (\cvec) - \regret_{\D_k} (\cvec) \right| \le \epsilon
$$
The proof is complete by noting that
\begin{align*}
\sup_{\cC \in \Delta \left(\reals_{\ge 0}^p \right), \, k \in [g]} \left| \expect_{\cvec \sim \cC} \left[ \regret_{G_k} (\cvec) \right] - \expect_{\cvec \sim \cC} \left[ \regret_{\D_k} (\cvec) \right] \right|
&\le \sup_{\cC \in \Delta \left(\reals_{\ge 0}^p \right), \, k \in [g]} \expect_{\cvec \sim \cC} \left[ \left|  \regret_{G_k} (\cvec)  - \regret_{\D_k} (\cvec)  \right| \right] \\
&= \sup_{\cvec \in \reals_{\ge 0}^p, \, k \in [g]} \left| \regret_{G_k} (\cvec) - \regret_{\D_k} (\cvec) \right|
\end{align*}
\end{proof}

\begin{proof}[Proof of Lemma~\ref{lem:pdim}]
	Our goal is to show that the class of functions $\cF_p$ can P-shatter at least
	$p$ points (or consumer risk limits). Let $\tau_1 < \dots < \tau_p \in [a,b]$
	be any sequence of increasing consumer risk limits. The high-level idea is to
	design a target $\gamma_i$ and a product $c_i$ for consumer $i$ so that the
	return for consumer $i$ is at least the target $\gamma_i$ if and only if the
	product $c_i$ is included (as opposed to a default product lower than any
	target). Given that we have $p$ products to choose, we can decide to include
	the product for each consumer or not independently, and therefore we can
	achieve all above/below target patterns to shatter the consumers.

	Formally,	define targets $\gamma_1, \dots, \gamma_p$ and a collection of
	candidate products $c_0, \dots, c_p$ so that
	\[
	r(c_0)
	< \gamma_1 < r(c_1) < r(\tau_1)
	< \gamma_2 < r(c_2) < r(\tau_2)
	< \dots
	< \gamma_p < r(c_p) < r(\tau_p).
	\]
	Note that this is always possible since the return function $r$ is strictly
	increasing on the interval $[a,b]$. For any product vector $\mathbf{d} \in
	\reals^p$ of product risk thresholds, we have that $f_{\mathbf{d}}(\tau_i)
	\geq \gamma_i$ if and only if there is some product $d_i$ such that $r(d_i)
	\in [\gamma_i, r(\tau_i)]$. Now, for any $T \subset [p]$, define a product
	risk threshold vector $\mathbf{d} \in \reals^p$ by $d_i = c_i$ if $i \in T$
	and $d_i = c_0$ if $i \not \in T$. By the above argument, and the definition
	of $c_i$ and $\gamma_i$, it follows that $f_{\mathbf{d}}(\tau_i) \geq
	\gamma_i$ if and only if $i \in T$. Therefore $\cF_p$ shatters $\tau_1, \dots,
	\tau_p$ and $\pdim(\cF_p) \geq p$.
\end{proof}

}

\end{document}